\documentclass[a4paper,12pt]{article}

\makeatletter
\newif\ifarxiv
\@ifclasswith{article}{twocolumn}{\arxivfalse}{\arxivtrue}
\makeatother

\ifarxiv
\usepackage{authblk}
\usepackage{geometry}
\else
\usepackage{iccv}
\usepackage{times}
\fi

\usepackage{epsfig}
\usepackage{graphicx}
\usepackage{amsmath}
\usepackage{amsthm}
\usepackage{amssymb}
\usepackage{MnSymbol}
\usepackage{graphicx}
\usepackage{makecell}
\usepackage{longtable}
\usepackage{multicol}
\allowdisplaybreaks
\usepackage{mwe}
\usepackage{subcaption} 
\usepackage{color}
\usepackage{algorithmic}
\usepackage{comment}
\usepackage{array}
\usepackage{floatrow}
\usepackage{tikz}
\usetikzlibrary{cd}
\graphicspath{{Figs/}}
\usepackage[pagebackref=true,breaklinks=true,colorlinks,bookmarks=false]{hyperref}
\ifarxiv\else
\iccvfinalcopy 
\fi

\newcommand{\anton}[1]{{\color{red} #1}}
\newcommand{\tim}[1]{{\color{green} #1}}
\newcommand{\TP}[1]{#1}
\newcommand{\KK}[1]{#1}

\newcommand{\hide}[1]{}
\ifarxiv \newcommand{\ie}{i.e.~} \fi
\ifarxiv \newcommand{\eg}{e.g.~} \fi
\newcommand{\github}{\url{https://github.com/timduff35/PLMP}}
\newcommand{\CC}{\mathbb{C}}
\newcommand{\FF}{\mathbb{F}}
\newcommand{\PP}{\mathbb{P}}
\newcommand{\QQ}{\mathbb{Q}}
\newcommand{\RR}{\mathbb{R}}
\newcommand{\ZZ}{\mathbb{Z}}
\newcommand{\GG}{\mathbb{G}}
\DeclareMathOperator{\SO}{\rm SO}
\newcommand{\bl}{{\mathbf l}}

\newcommand{\cY}{\mathcal{Y}}

\newcommand{\ratmap}{\dashedrightarrow}
\newcommand{\GB}{Gr\"obner basis}
\newcommand{\GBs}{Gr\"obner bases}

\newcommand{\Inc}{\mathrm{Inc}}

\newcommand{\XCa}{\mathcal{X}}

\newcommand{\ICa}{\mathcal{I}}

\newcommand{\YCa}{\mathcal{Y}}
\newcommand{\CCa}{\mathcal{C}}
\newcommand{\xx}[1]{\left[#1\right]_\times}

\newcommand{\PplI}{\XCa_{p, l, \ICa}}

\newcommand{\Cm}{\CCa_m}
\newcommand{\PLP}{{p, l, \ICa, m}}
\newcommand{\YplIm}{\YCa_{p, l, \ICa, m}}
\DeclareMathOperator{\rank}{rank}
\newtheorem{theorem}{Theorem}

\newtheorem{corollary}{Corollary}
\newtheorem{lemma}{Lemma}
\newtheorem{algorithm}{Algorithm}
\theoremstyle{definition}
\newtheorem{definition}{Definition}
\newtheorem{remark}{Remark}
\newtheorem{example}{Example}
\ifarxiv
\else
    
\fi
\ifarxiv
\else \ificcvfinal\fi
\title{PLMP - Point-Line Minimal Problems in Complete Multi-View Visibility}
\ifarxiv 
\author[1]{Timothy Duff}
\author[2]{Kathl\'en Kohn}
\author[3]{Anton Leykin}
\author[4]{Tomas Pajdla}
\affil[1,3]{Georgia Tech}
\affil[2]{KTH\footnote{KTH Royal Institute of Technology in Stockholm}}
\affil[4]{CIIRC, CTU in Prague\footnote{Czech Institute of Informatics, Robotics and Cybernetics, Czech Technical University in Prague.}
}
\else
\author{Timothy Duff\\
School of Mathematics, Georgia Tech
\and
Kathl\'en Kohn\\
KTH\\
\and
Anton Leykin\\
School of Mathematics, Georgia Tech\\
\and
Tomas Pajdla\\
CIIRC - Czech Technical University in Prague\thanks{CIIRC - Czech Institute of Informatics, Robotics and Cybernetics}} 
\fi
\begin{document}
\maketitle
\ifarxiv
\else \ificcvfinal\thispagestyle{empty}\fi
\begin{abstract}
\noindent We present a complete classification of all minimal problems for generic arrangements of points and lines completely observed by calibrated perspective cameras. We show that there are only 30 minimal problems in total, no problems exist for more than 6 cameras, for more than 5 points, and for more than 6 lines. We present a sequence of tests for detecting minimality starting with counting degrees of freedom and ending with full symbolic and numeric verification of representative examples. For all minimal problems discovered, we present their algebraic degrees, \ie the number of solutions, which measure their intrinsic difficulty. 
It shows how exactly the difficulty of problems grows with the number of views. Importantly, several new minimal problems have small degrees that might be practical in image matching and 3D reconstruction.
\end{abstract}
\section{Introduction}
\noindent Minimal\ifarxiv\else\footnote{We are grateful to ICERM (NSF DMS-1439786 and the Simons Foundation grant 507536) for the hospitality (09/2018 -- 02/2019), where most ideas for this project were developed. We thank the many research visitors at ICERM for fruitful discussions on minimal problems. Research of T.~Duff and A.~Leykin is supported in part by NSF DMS-1719968. T.~Pajdla was supported by the European Regional Development Fund under the project IMPACT (reg. no. CZ.02.1.01/0.0/0.0/15 003/0000468).}\fi problems~\cite{Nister-5pt-PAMI-2004,Stewenius-ISPRS-2006,kukelova2008automatic,larsson2017efficient,Larsson-Saturated-ICCV-2017,larsson2017making,kukelova2017clever,Larsson-CVPR-2018} play an important role in 3D reconstruction~\cite{Snavely-SIGGRAPH-2006,snavely2008modeling,schoenberger2016sfm}, image matching~\cite{rocco2018neighbourhood}, visual odometry~\cite{Nister04visualodometry,Alismail-odometry} and visual localization~\cite{taira2018inloc,Sattler-PAMI-2017,svarm2017city}. Many minimal problems have been described and solved and new minimal problems are constantly appearing. In this paper, we present a step towards a complete characterization of all minimal problems for points, lines and their incidences in calibrated multi-view geometry. This is a grand challenge, especially when dealing with partial visibility due to occlusions and missing detections. Here we provide a complete characterization for the case of complete multi-view visibility.
\KK{Informally, a minimal problem is a 3D reconstruction problem recovering camera poses and world coordinates from given images such that random input instances have a finite positive number of solutions.}
\ifarxiv
\begin{figure}[h]
    \centering
    \includegraphics[width=0.49\linewidth]{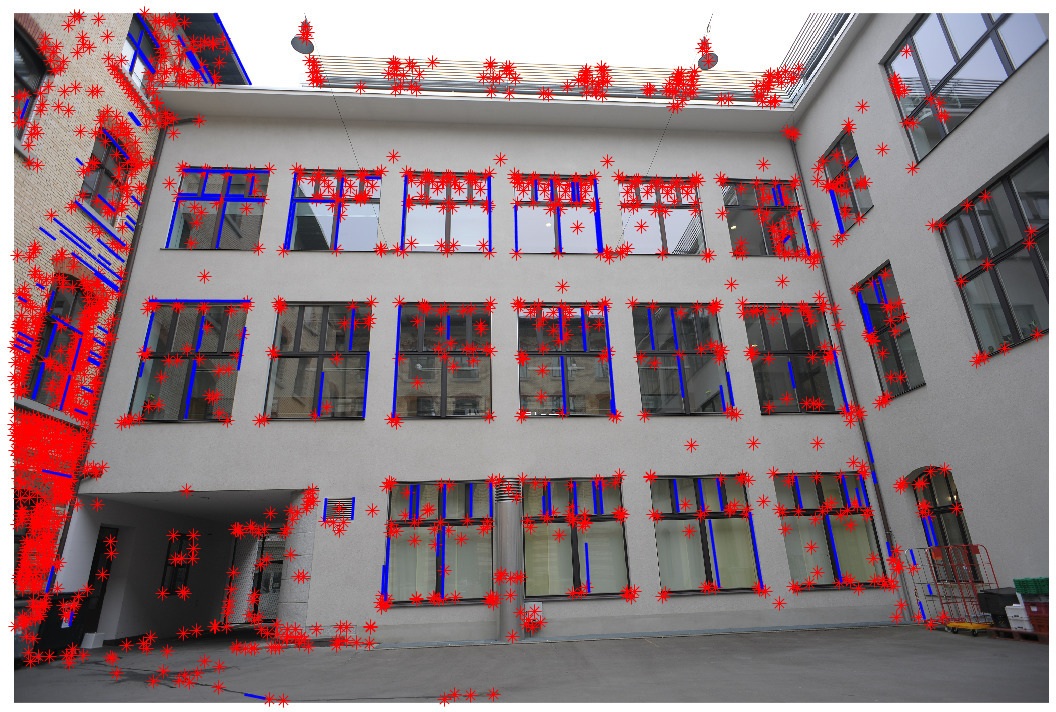}
    \includegraphics[width=0.49\linewidth]{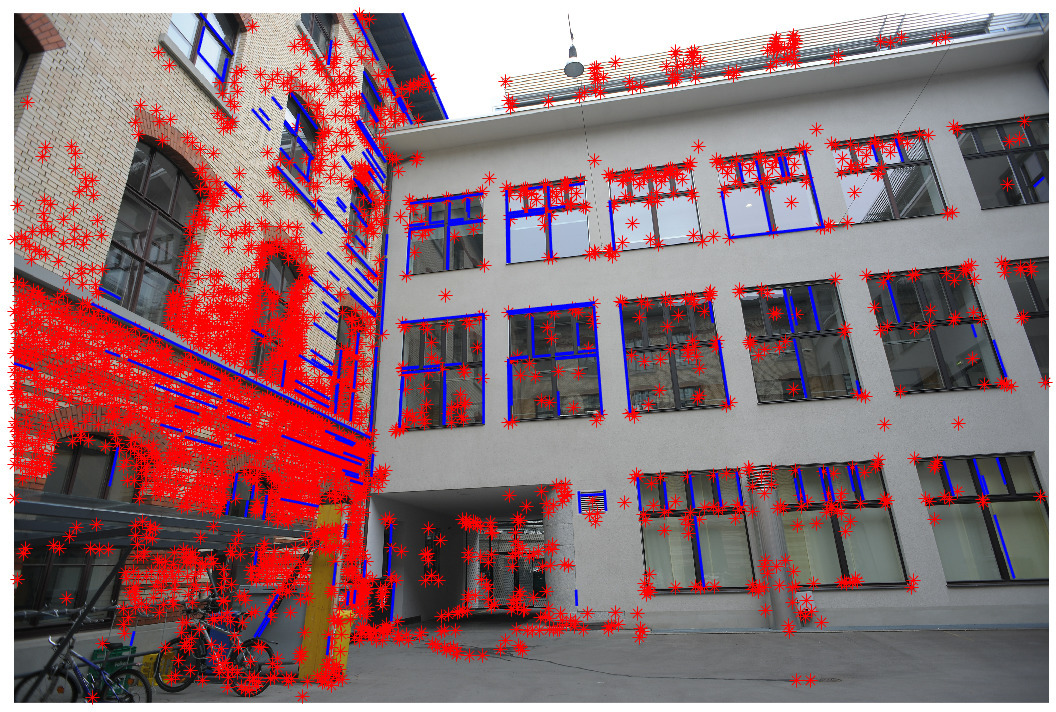}\newline
    \includegraphics[width=0.24\linewidth]{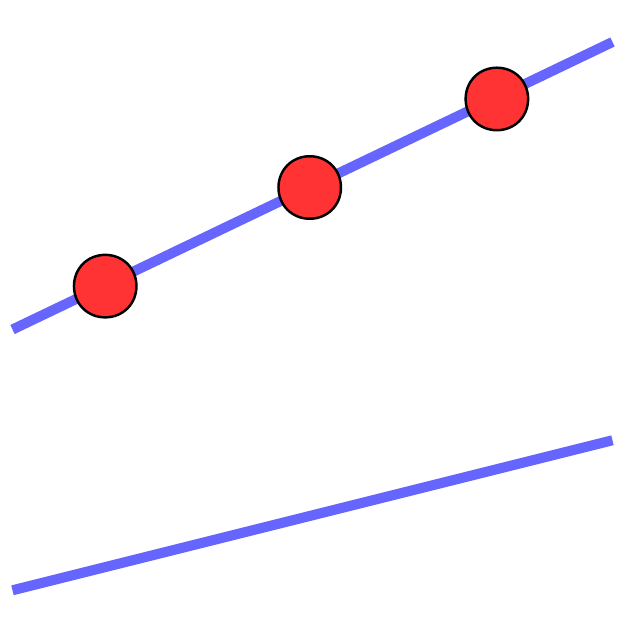}
    \includegraphics[width=0.24\linewidth]{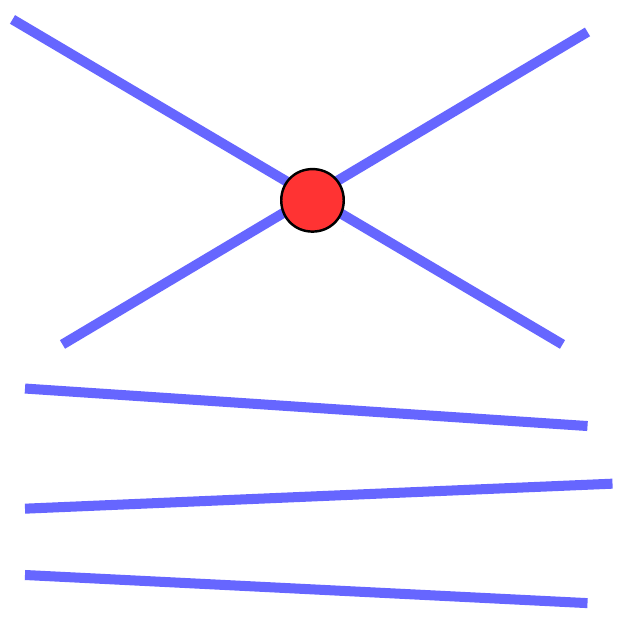}
    \includegraphics[width=0.24\linewidth]{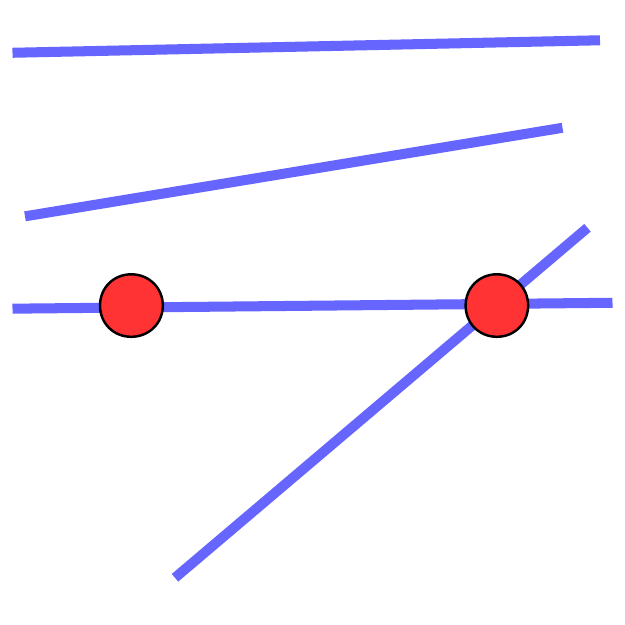}
    \includegraphics[width=0.24\linewidth]{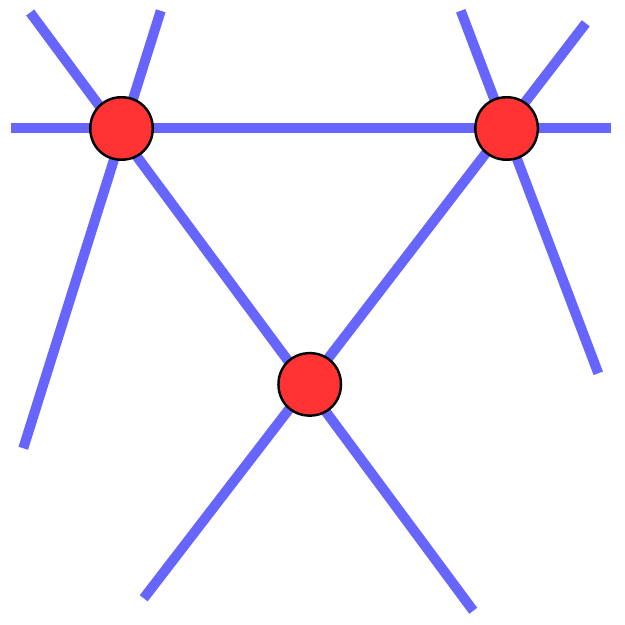}
    \caption{(1-st row) Points (red) and lines (blue) get detected independently as well as in arrangements with points incident to lines~\cite{Miraldo-ECCV-2018}. (2-nd row) Examples of some interesting arrangements of points and lines providing new minimal problems. See Tab.~\ref{tab:balanced} for the complete classification of minimal problems for points, lines and their arrangements in multiple images with complete multi-view visibility.}
    \label{fig:teaser-figure}
\end{figure}
\else
\begin{figure}
    \centering
    \includegraphics[width=0.49\linewidth]{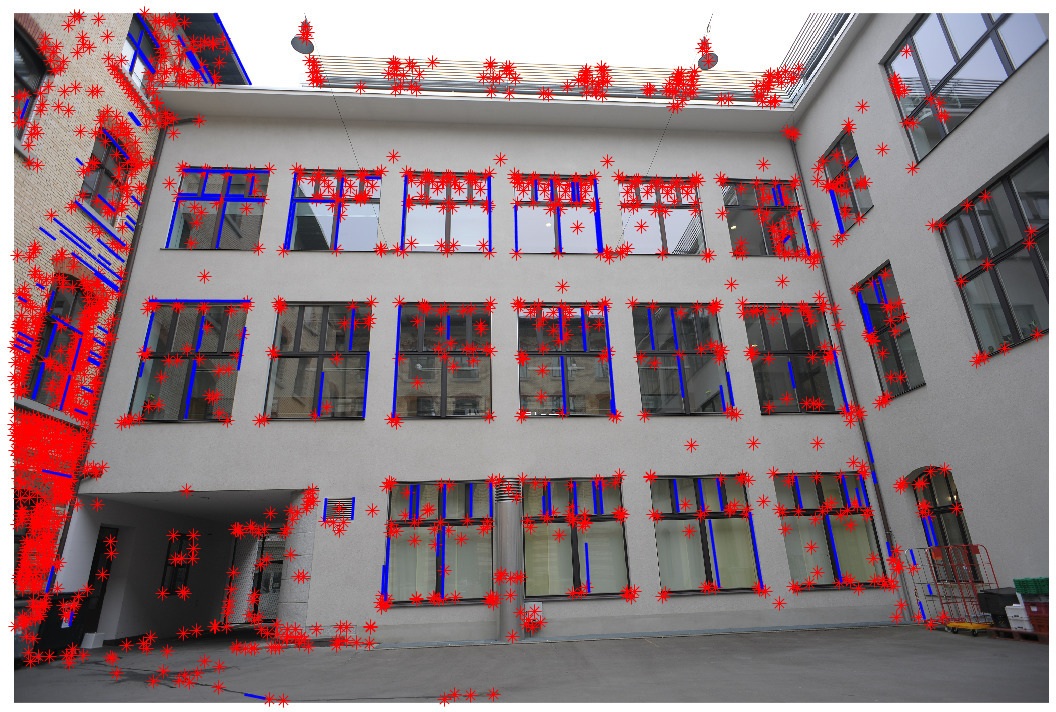}
    \includegraphics[width=0.49\linewidth]{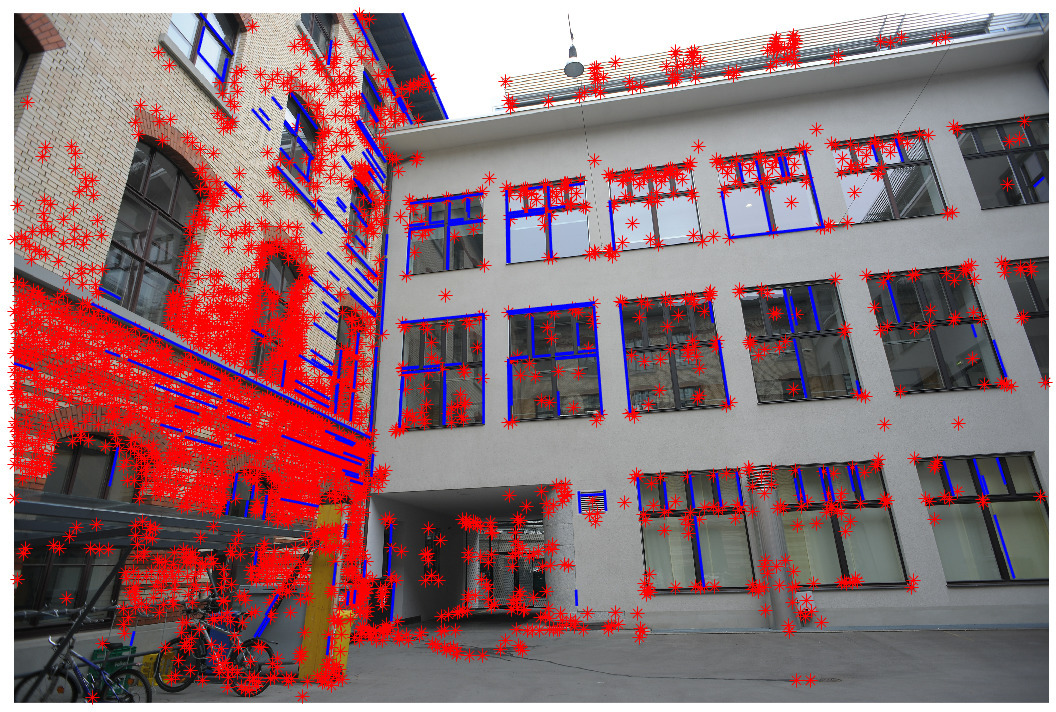}\newline
    \includegraphics[width=0.24\linewidth]{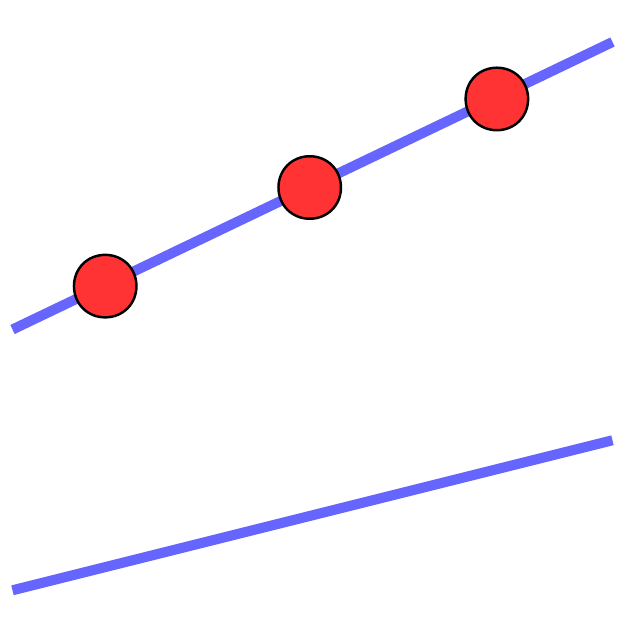}
    \includegraphics[width=0.24\linewidth]{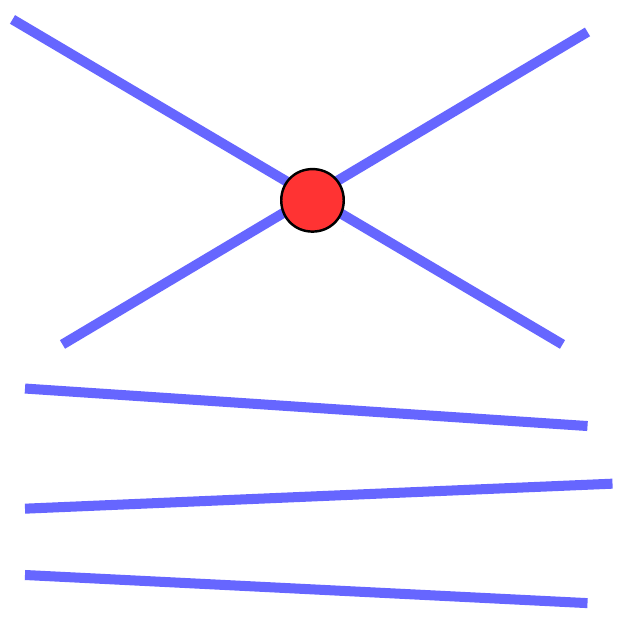}
    \includegraphics[width=0.24\linewidth]{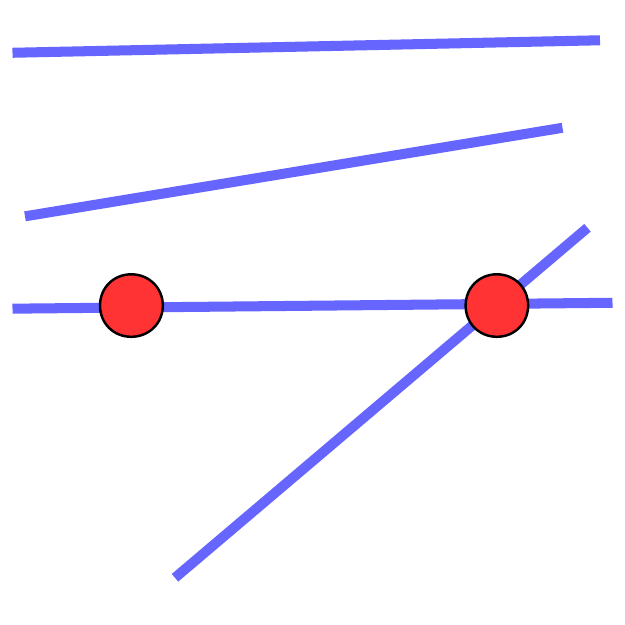}
    \includegraphics[width=0.24\linewidth]{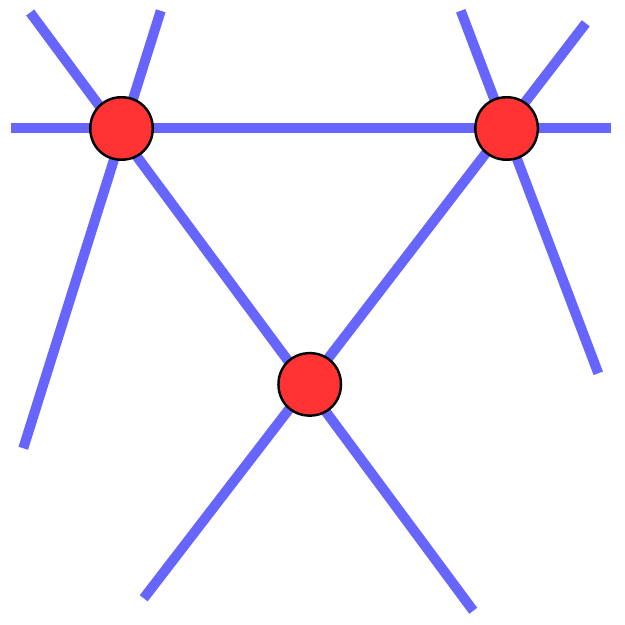}
    \caption{(1-st row) Points (red) and lines (blue) get detected independently as well as in arrangements with points incident to lines~\cite{Miraldo-ECCV-2018}. (2-nd row) Examples of some interesting arrangements of points and lines providing new minimal problems. See Tab.~\ref{tab:balanced} for the complete classification of minimal problems for points, lines and their arrangements in multiple images with complete multi-view visibility.}
    \label{fig:teaser-figure}
\end{figure}
\mbox{}\\[3pt] 
\fi

\noindent {\bf Contribution~} We give a complete classification of minimal problems for generic arrangements of points and lines, including their incidences, completely observed by any number of calibrated perspective cameras. We consider calibrated scenarios since it avoids many degeneracies~\cite{HZ-2003}. 

We show that there are exactly $30$ minimal point-line problems (up to an arbitrary number of lines in the case of two views) when considering complete visibility (Tab.~\ref{tab:balanced}). In particular, there is no such minimal problem for seven or more cameras. For $6$, $5$, $4$ and $3$ cameras, there are $1$, $3$, $6$ and $17$ minimal problems, respectively. For two views, there are three combinatorial constellations of five points which yield minimal problems. We observe that each minimal point-line problem has at most five points and at most six lines (except for arbitrarily many lines in the case of two views.) Problems $5000_2$~\cite{Nister-5pt-PAMI-2004},
\KK{$3200_3$~\cite{Faugeras-IJPRAI-1988,Malis07-RR6303},}
$3010_0, 1040_0$~\cite{Kileel-MPCTV-2016} have been known before, all other 26 minimal problems in Tab.~\ref{tab:balanced}, as far as we know, are new.

For each minimal problem, we compute its \emph{algebraic degree} which is its number of solutions over the complex numbers for generic images. This degree measures the intrinsic difficulty of a minimal problem. We observe how this degree generally grows with the number of cameras, but we also found several minimal problems with small degrees (32, 40 and 64), which might be practical in image matching and 3D reconstruction~\cite{schoenberger2016sfm}.

We consider generic minimal problems, \ie the problems that have a finite number of complex solutions and are generic in the sense that random noise in image measurements does not change the number of solutions. For instance, the classical problem of five points in two views~\cite{Nister-5pt-PAMI-2004} is minimal and one can add arbitrarily many lines to the arrangement in $3$-space; as long as it contains five points in sufficiently generic position, it is still minimal. On the other hand, the problem of four points in three views~\cite{Longuet-Higgins-IVC-1992,
Holt-PAMI-1995,Nister-IJCV-2006,QuanTM2006} is overconstrained when all measurements and equations are used. It becomes inconsistent for noisy image measurements. Thus it is not a minimal problem for us.  

We assume complete visibility, \ie all points and lines are observed in all images and all observed information is used to formulate minimal problems. Complete point-line incidence correspondences arise when, \eg, SIFT point features~\cite{Lowe04IJCV} are considered together with their orientation, lines are constructed from matched affine frames~\cite{Matas-ICPR-2002}, or obtained as simultaneous point and line detections~\cite{Miraldo-ECCV-2018}, Fig.~\ref{fig:teaser-figure}\footnote{Real images from~\cite{Miraldo-ECCV-2018} by courtesy of S.~Ramalingam.}. On the other hand, we do not cover cases that require partial visibility, \eg, 3 PPP + 1 PPL in~\cite{Kileel-MPCTV-2016}. Full account for partial visibility is a much harder task and will be addressed in the future.

We explicitly model point-line incidences. Several lines may be incident to a single point (n-quiver) and several points may be dependent by lying on a single line. We assume that such relations are not broken by image noise since they are constructed by the feature detection process. 

Our problem formulation uses direct geometrical determinantal constraints as in~\cite{Johansson-ICVGIP-2002,Oskarsson-IVC-2004}, not multi-view tensors, since it works for any number of cameras and can model point-line incidences. On the other hand, our formulation is not the most economical for the minimal problem with five independent points in two views ($5000_2$ in~Tab.~\ref{tab:balanced}). This problem has degree $20$ in our formulation while the degree is only 10 when reformulated~\cite{Nister-CVPR-2007} as a problem of finding the essential matrix.\footnote{\TP{Proving that the minimal problems with three or more views cannot be reformulated in a way that decreases the degrees that we report, or finding reformulations, may require more advanced algebraic techniques and presents a challenge for the future research.}} 
\paragraph{Structure of the paper} The paper is organized as follows. We review previous work in Sec.~\ref{sec:previous-work}. Section~\ref{sec:concepts} defining main concepts is followed by problem specification in Sec.~\ref{sec:problem-specification}. All candidates for minimal problems satisfying balanced counts of degrees of freedom are identified in Sec.~\ref{sec:balanced-problems}. Section~\ref{sec:equations} presents our parameterization of the problems for computational purposes. Procedures for checking the minimality and computing the degrees using symbolic and numerical methods from algebraic geometry are presented in Sec.~\ref{sec:degrees-computation}.
\section{Previous work}\label{sec:previous-work}
\noindent Here we review the most relevant work for point-line incidences and minimal problems. See~\cite{Larsson-CVPR-2018, kukelova2017clever, Kileel-MPCTV-2016} for references on minimal problems in general.
Using correspondences of non-incident points and lines in three uncalibrated views was considered in works on the trifocal tensor~\cite{Hartley-IJCV-1997}. The early work on point-line incidences~\cite{Johansson-ICVGIP-2002} introduced n-quivers, \ie points incident with n lines in uncalibrated views, and studied minimal problems arising from three 1-quivers in three affine views and three 3-quivers in three perspective views, as well as the overconstrained problem of four 2-quivers in three views.

\TP{Uncalibrated multi-view constraints for points, lines and their incidences appeared in~\cite{MaHVKS-IJCV-2004}.} In~\cite{Oskarsson-IVC-2004},  non-incident points and lines in uncalibrated views were studied and four points and three lines in three views, two points and six lines in three views, and nine lines in three views cases were presented. The solver for the latter case has recently appeared in~\cite{Larsson-Syzygy-CVPR-2017}. Absolute pose of cameras with unknown focal length from 2D-3D quiver correspondences has been solved in~\cite{Kuang-ICCV-2013} for two points and 1-quiver, for one 1-quiver and one 2-quiver, and for four lines.
In \cite{Fabbri-IJCV-2016}, an important case, when lines incident to points arise from tangent lines to curves, is presented. It motivates the case with three points and tangent lines at two points (case $3002_1$ in Tab.~\ref{tab:balanced}). Work~\cite{Miraldo-ECCV-2018} presents several minimal problems for generalized camera absolute pose computation from 2D-3D correspondences of non-incident points and lines with focus on cases when a closed-form solution could be found. \TP{In~\cite{Elqursh-CVPR-2011,SalaunMM-ECCV-2016}, parallelism and perpendicularity of lines in space were exploited to find calibrated relative pose from lines and points.} 
Recent work~\cite{Zhao-PAMI-2019} investigates calibrated relative camera pose problems from two views with 2-quivers 
with known angles between the 3D lines generating the quivers. Minimal problems for finding the relative pose from three such correspondences for the generic as well as several more specific cases is derived. Our closest generalization of this result is that one can obtain calibrated relative pose of three cameras from one 2-quiver and two independent points in three views without knowing angles in 3D. Recently, minimal problems were constructed for local multi-features including lines incident to points as well as more complex features~\cite{Barath-CVPR-2017,Barath-CVPR-2018,Barath-TIP-2018}. They build on SIFT directions~\cite{Lowe04IJCV} or more elaborate local affine features~\cite{Matas-ICPR-2002} to reduce the number of samples needed in RANSAC~\cite{Raguram11IJCV}  to verify tentative matches. 

\noindent {\bf The most relevant previous work~} Recent theoretical results~\cite{JoswigKSW16,AgarwalLST17,AholtO14,Aholt-1107-2875,Trager-PhD-2018,Ponce-IJCV-2016} made steps towards characterizing some of the classes of minimal problems. The most relevant work~\cite{Kileel-MPCTV-2016} provided a classification for three calibrated views that can be formulated using linear constraints on the trifocal tensor~\cite{HZ-2003}. In~\cite{Kileel-MPCTV-2016}, 66 minimal problems for three calibrated views were presented and their algebraic degrees computed. The lowest degree 160 has been observed for one PPP and four PPL constraints while the highest degree 4912 has been observed for 11 PLL constraints. 
Out of 66 problems in~\cite{Kileel-MPCTV-2016} the ones that can be modeled with complete visibility are (1 PPP + 4 LLL) and (3 PPP + 1 LLL). These two minimal problems appear as $1040_0$ and $3010_0$ in  Tab.~\ref{tab:balanced}. The other 15 minimal problems in three views that we discovered do not appear in~\cite{Kileel-MPCTV-2016} since the point-line incidences were not considered.
\section{Notation and concepts}\label{sec:concepts}
\noindent We use nomenclature from~\cite{HZ-2003}.
Points and lines in space are in the projective space $\PP^3$, image points are in $\PP^2$ and are represented by homogeneous coordinates. We consider the Grassmannians $\GG_{1,3}$ and  $\GG_{1,2}$ which are the spaces of lines in $\PP^3$ and $\PP^2,$ respectively. $\SO(3)$ stands for the special orthogonal group, \ie rotations, defined algebraicly as $3\times 3$ matrices $R$ such that $R R^\top = I$, $\det{R}=1 $.
All is considered over an arbitrary field $\FF$ unless explicitly specified. Coefficients of equations originate from the field $\QQ$ of rational numbers. Solutions of the equations are in the field $\CC$ of complex numbers. We carry out symbolic computations in a finite field $\ZZ_p$ for a prime $p$ for the sake of exactness and computational efficiency. Numerical algorithms use floating point to approximate complex numbers.
\ifarxiv
\else \footnote{
See Sec.~Notation and Concepts of Supplementary material for more details.}
\fi

\section{Problem Specification}\label{sec:problem-specification}
\noindent Our main result applies to problems in which points, lines, and point-line incidences are observed. 
We first introduce a \emph{point-line problem} as a tuple $(\PLP)$ specifying that $p$ points and $l$ lines in space, which are incident according to a given incidence relation $\mathcal{I}\subset \{ 1, \ldots , p \} \times \{ 1, \ldots , l \}$ (\ie $(i,j)\in \mathcal{I}$ means that the $i$-th point is on the $j$-th line) are projected to $m$ views. So a point-line problem captures the numbers of points, lines and views as well as the incidences between points and lines. We will model intersecting lines by requiring that each intersection point of two lines has to be one of the $p$ points in the point-line problem. Throughout this article we will only consider incidence relations which can be realized by a point-line arrangement in $\PP^3$.
In particular, two distinct lines cannot be incident to the same two distinct points. In addition, we will always assume that the incidence relation $\ICa$ is complete in the sense that every incidence which is automatically implied by the incidences in $\ICa$ must also be contained in $\ICa$.
An \emph{instance} of a point-line problem is specified by the following data:

(1) A point-line arrangement in space consisting of $p$ points $X_1,\ldots,X_p$ and $l$ lines $L_1,\ldots,L_l$ in $\PP^3$ which are incident exactly as specified by $\mathcal{I}\subset \{ 1, \ldots , p \} \times \{ 1, \ldots , l \}$.
Hence, the point $X_i$ is on the line $L_j$ if and only if $(i,j)\in \mathcal{I}$. We write 
\[
\PplI = \left\{ (X, L) \in \left( \PP^3 \right)^p  \times  \left( \GG_{1,3} \right)^{l} \mid \forall (i,j) \in \mathcal{I}\, \colon X_i \in L_j   \right\}
\]
for the associated \emph{variety of point-line arrangements}.
Note that this variety also contains degenerate arrangements,
where not all points and lines have to be pairwise distinct
or where there are more incidences between points and lines than those specified by $\mathcal{I}$.

(2) A list of $m$ calibrated cameras which are represented by matrices \[
P_1 = [R_1 \mid t_1], \ldots, P_m = [R_m \mid t_m]
\]
with $R_1, \ldots, R_m \in \SO (3)$ and $t_1, \ldots, t_m \in \FF^{3}$. 

(3) The \emph{joint image} consisting of the projections $x_{v,1}, \ldots  , x_{v, p} \in \PP^2$ of the points $X_1, \ldots, X_p$ and the projections $\ell_{v,1}, \ldots , \ell_{v,l} \in \GG_{1,2}$ of the lines $L_1, \ldots,L_l$ by the cameras $P_1,\ldots,P_m$
to the $v = 1,\ldots,m$ views. We write
\ifarxiv
$$
\YplIm = \left\{ (x, \ell) \in \left(\PP^2 \right)^{m\, p} \times \left( \GG_{1,2} \right)^{m \, l} \;\middle\vert\; \forall v=1,\ldots,m \;   
\forall (i,j) \in \mathcal{I}:   x_{v, i} \in \ell_{v, j} \right\} 
$$
\else
\begin{multline*}
\YplIm = \left\{ (x, \ell) \in \left(\PP^2 \right)^{m\, p} \times \left( \GG_{1,2} \right)^{m \, l} \;\middle\vert\; \forall v=1,\ldots,m  \right. \\ 
\left. 
\phantom{\times \left( \GG_{1,2} \right)^{3 \, l}}
\forall (i,j) \in \mathcal{I}:   x_{v, i} \in \ell_{v, j} \right\} 
\end{multline*}
\fi
for the \emph{image variety}
which consists of all $m$-tuples of two-dimensional point-line arrangements which satisfy the incidences specified by $\ICa$. 
%
\hide{
\item  In each view $v=1,2,3,$ we observe $p$ points $x_{v,1}, \ldots  , x_{v, p} \in \PP^2$ and $l$ lines $\ell_{v,1}, \ldots , \ell_{v,l} \in \GG_{1,2} .$ \tim{This seems consistent with Tomas had written before}
\item There is a fixed set of incidence conditions $$\mathcal{I}\subset \{ 1, \ldots , p \} \times \{ 1, \ldots , l \}.$$ The incidence $(i,j)\in \mathcal{I}$ implies we observe $x_{v, i} \in \ell_{v, j}$ for each $v.$ \tim{This is not quite compatible with Kathlen's notation. I think that talking about the line arrangement isn't necessary in the setting of this paper.}\anton{We should define what ``feasible'' incidence conditions are somewhere.}
\item To represent problems in a minimal manner, we also require that for each $j$ there exists at most one $i$ with $(i,j) \in \mathcal{I}.$ Stated simply: a line with two visible points is already visible.
}

Given a joint image, we want to recover an arrangement in space and cameras yielding the given joint image. We refer to a pair of such an arrangement and such a list of $m$ cameras as a \emph{solution} of the point-line problem for the given joint image.
We note that an $m$-tuple in $\YplIm$ does not necessarily admit a solution, 
\ie a priori it does not have to be a joint image of a common point-line arrangement in $3$-space. 

To fix the arbitrary space coordinate system~\cite{HZ-2003}, we set $P_1 = [I\,|\,0]$ and the first coordinate of $t_2$ to $1$. Hence, our \emph{camera configurations} are parameterized by
\ifarxiv
$$
\Cm = \left\{ (P_1, \ldots, P_m) \in \left(\FF^{3\times 4}\right)^{m} \;\middle\vert\; P_i = [R_i \mid t_i], \,
R_i \in \SO (3), t_i \in \FF^3,\, R_1 = I, t_1 = 0, t_{2,1}=1
\right\}.
$$
\else
\begin{multline*}
\Cm = \left\{ (P_1, \ldots, P_m) \in \left(\FF^{3\times 4}\right)^{m} \;\middle\vert\; P_i = [R_i \mid t_i], \right.
\\ \hspace*{-3mm}
\left.
\phantom{\left(\FF^{3\times 4}\right)^{m-1}}
R_i \in \SO (3), t_i \in \FF^3,\, R_1 = I, t_1 = 0, t_{2,1}=1
\right\}.
\end{multline*}
\fi
We will always assume that the camera positions in an instance of a point-line problem are sufficiently generic such that the following three natural conditions are satisfied for each camera:
Firstly, two distinct lines or points in the given arrangement in $3$-space are viewed as distinct lines or points.
Secondly, a point and a line in the space arrangement, which are not incident in $3$-space, are viewed as non-incident.
Thirdly, three non-colinear points in the space arrangement are viewed as non-colinear points.

We say that a point-line problem is \emph{minimal} if a generic image tuple in $\YplIm$ has a nonzero finite number of solutions.
We may phrase this definition formally:
\begin{definition}
Let $\Phi_\PLP: \PplI \times \Cm \dashrightarrow \YplIm$
denote the \emph{joint camera map},
which sends a point-line arrangement in space and $m$ cameras
to the resulting joint image. We say that the point-line problem $(\PLP)$ is \emph{minimal} if

\ifarxiv
\begin{itemize}
    \item  $\Phi_\PLP$ is a \emph{dominant map}\footnote{Dominant maps are analogs of surjective maps in birational geometry.}, \ie{} a generic element $(x, \ell)$ in $\YplIm$ has a solution, so $\Phi_\PLP^{-1} (x, \ell) \neq \emptyset$, and
    \item the preimage $\Phi_\PLP^{-1} (x, \ell)$ of a generic element $(x, \ell)$ in $\YplIm$ is finite.
\end{itemize}
\else
\noindent $\bullet\  \Phi_\PLP$ is a \emph{dominant map}\footnote{Dominant maps are analogs of surjective maps in birational geometry.}, \ie{} a generic element $(x, \ell)$ in $\YplIm$ has a solution, so $\Phi_\PLP^{-1} (x, \ell) \neq \emptyset$, and

\noindent $\bullet\ $the preimage $\Phi_\PLP^{-1} (x, \ell)$ of a generic element $(x, \ell)$ in $\YplIm$ is finite.
\fi
\end{definition}

\begin{remark}
\label{remark:perturbation}
For a given a minimal problem $(\PLP ),$ the joint camera map $\Phi_\PLP$ maps $\PplI \times \Cm$ onto a constructible subset of $\YplIm$ of the same dimension. Given a solution for a generic joint image $(x,\ell )$ when $\FF = \RR$ or $\CC ,$ there exists a ball around $(x,\ell),$ say $B_\epsilon (x,\ell ),$ and for each solution $(X,C)\in \Phi_\PLP^{-1} (x,\ell )$ a ball $B_\delta (X,C ) $ such that $$\Phi_\PLP \left(B_\delta (X, C) \right) \subset B_\epsilon (x, \ell ).$$ In this sense, we may deduce that solutions to minimal problems are stable under perturbation of the data. 

The joint camera map $\Phi_\PLP$ reflects that we want to recover world points and lines as well as camera poses from a given joint image.
In the case of complete visibility, this is equivalent to only recovering camera poses.
We formalize this observation in Lemma~\ref{lem:restrictedIncidence} and Corollary~\ref{cor:restrictedIncidence}.
\end{remark}
\noindent Over the complex numbers, the cardinality of the preimage $\Phi_\PLP^{-1} (x, \ell)$ is the same for every \emph{generic} joint image $(x, \ell)$ of a minimal point-line problem $(\PLP)$.
We refer to this cardinality as the \emph{degree} of the minimal problem. Our goal is to list \emph{all} minimal point-line problems and to compute their degrees.  For this, we pursue the following strategy:

\textbf{Step 1:} A classical statement from algebraic geometry states for a dominant map $\varphi: X \dashrightarrow Y$ from a variety $X$ to another variety $Y$ that the preimage $\varphi^{-1}(y)$ of a generic point $y$ in $Y$ has dimension $\dim(X) - \dim (Y).$ 
When $\varphi$ is a linear map between linear spaces, this is simply the rank-nullity theorem of linear algebra.  
As the generic preimage of the joint camera map $\Phi_\PLP$ associated to a minimal point-line problem $(\PLP)$ is zero-dimensional, we see that every minimal point-line problem must satisfy the equality $\dim (\PplI \times \Cm) = \dim (\YplIm)$. This motivates the following definition.
\begin{definition}
We say that a point-line problem $(\PLP)$ is \emph{balanced} if $\dim (\PplI \times \Cm) = \dim (\YplIm)$.
\end{definition}
\noindent As we have now established that all minimal point-line problems are balanced, we classify all balanced point-line problems in Sec.~\ref{sec:balanced-problems}. We will see that there are only finitely many such problems, explicitly given in Tab.~\ref{tab:balanced}, up to arbitrarily many lines in the case of two views; see Remark~\ref{rem:2views}.

\textbf{Step 2:}
The classical statement from algebraic geometry mentioned above further implies that a balanced point-line problem $(\PLP)$ is minimal if and only if its joint camera map $\Phi_\PLP$ is dominant. 
Hence, to determine the exhaustive list of all minimal point-line problems, we only have to check for each balanced point-line problem in Tab.~\ref{tab:balanced} if its joint camera map is dominant. 
We perform this check computationally, as described in Sec.~\ref{sec:degrees-computation}.

\textbf{Step 3:}
Finally, we use symbolic and numerical computations to calculate the degrees of the minimal point-line problems. We describe these computations in Sec.~\ref{sec:degrees-computation}.
\newcommand{\LTW}{0.070}
\newcommand{\LTC}{0.007}
\newcommand{\LTS}{5.8pt}
\addtolength{\tabcolsep}{-\LTS} 
\begin{table*}[t]
\caption[All (modulo additional lines in two views) balanced point-line problems.]{All balanced point-line problems, modulo adding arbitrarily many lines to the problems with $2$ views. Some problems are not uniquely identified by their vector $(p^\mathrm{f},p^\mathrm{d},l^\mathrm{f},l^\mathrm{a})$. To make the identification unique, we extend the vector by a subscript $\alpha$, which is the maximum number of lines adjacent to the same point in the case of at least three views or the maximum number of points on a common line in the case of two views. Degrees marked with ${}^*$ have been computed with numerical methods, the others with symbolic algorithms; see Section~\ref{sec:degrees-computation}. \TP{Problem $3200_3$ has all five points in a single 3D plane: it corresponds to the calibrated homography relative pose computation~\cite{Faugeras-IJPRAI-1988,Malis07-RR6303}; see Supplementary Material.}}
\label{tab:balanced}
\begin{tabular}{|@{\hskip 2pt}c@{\hskip 2pt}|ccccccccccccc|}
\hline
$m$ views &6&6&6 &5&5&5 &4&4&4&4&4&4&4 \\
$p^\mathrm{f}p^\mathrm{d}l^\mathrm{f}l^\mathrm{a}_\alpha$&
$1021_1$&$1013_3$&$1005_5$&$2011_1$&$2003_2$&$2003_3$&$1030_0$&$1022_2$&$1014_4$&$1006_6$&$3001_1$&$2110_0$&$2102_1$
\\
\parbox[b]{1.1cm}{$(p,l,\ICa)$\\[\LTC\textwidth]} 
&\includegraphics[width=\LTW\textwidth]{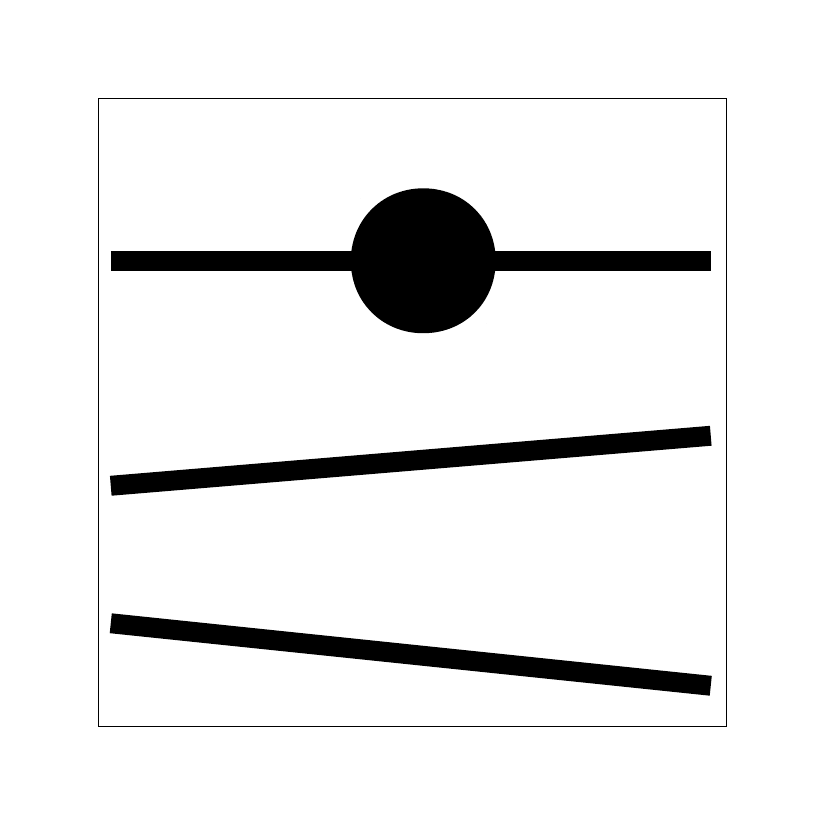} 
&\includegraphics[width=\LTW\textwidth]{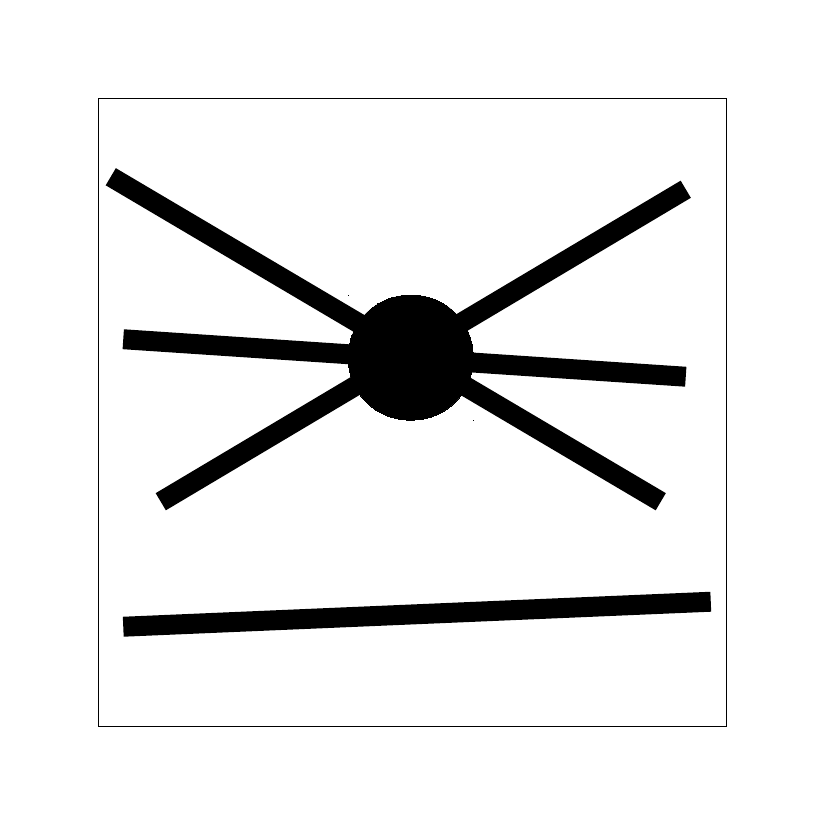} 
&\includegraphics[width=\LTW\textwidth]{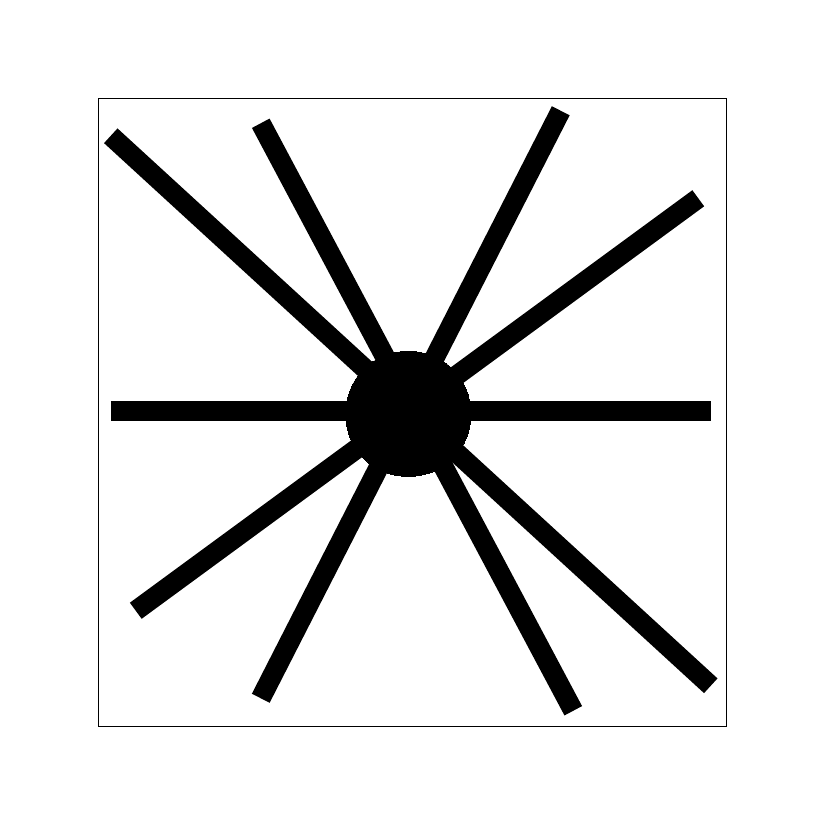} 
&\includegraphics[width=\LTW\textwidth]{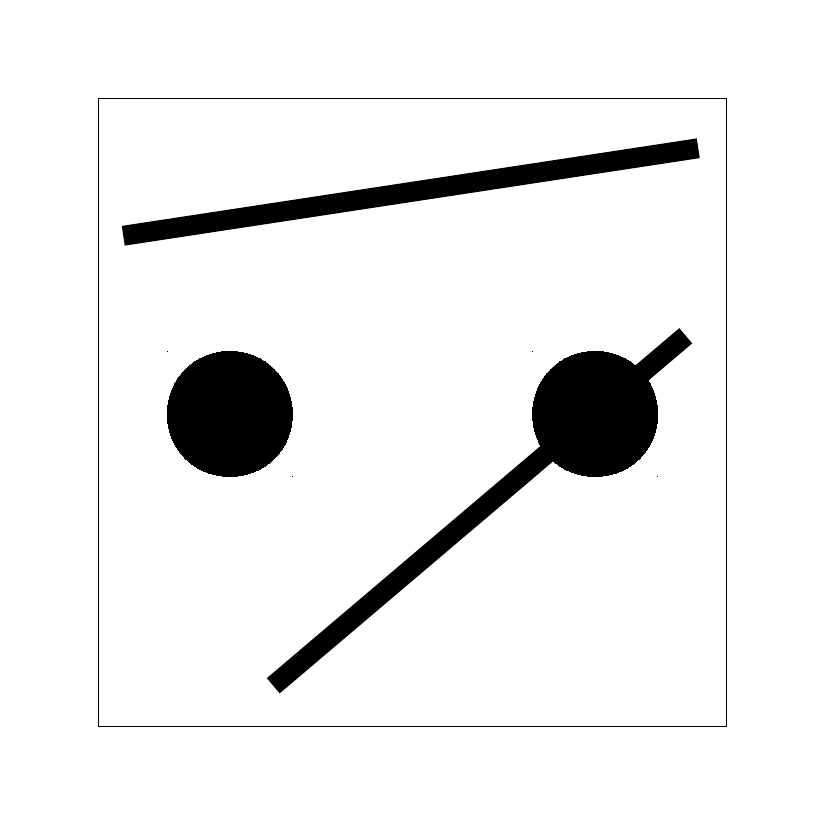} 
&\includegraphics[width=\LTW\textwidth]{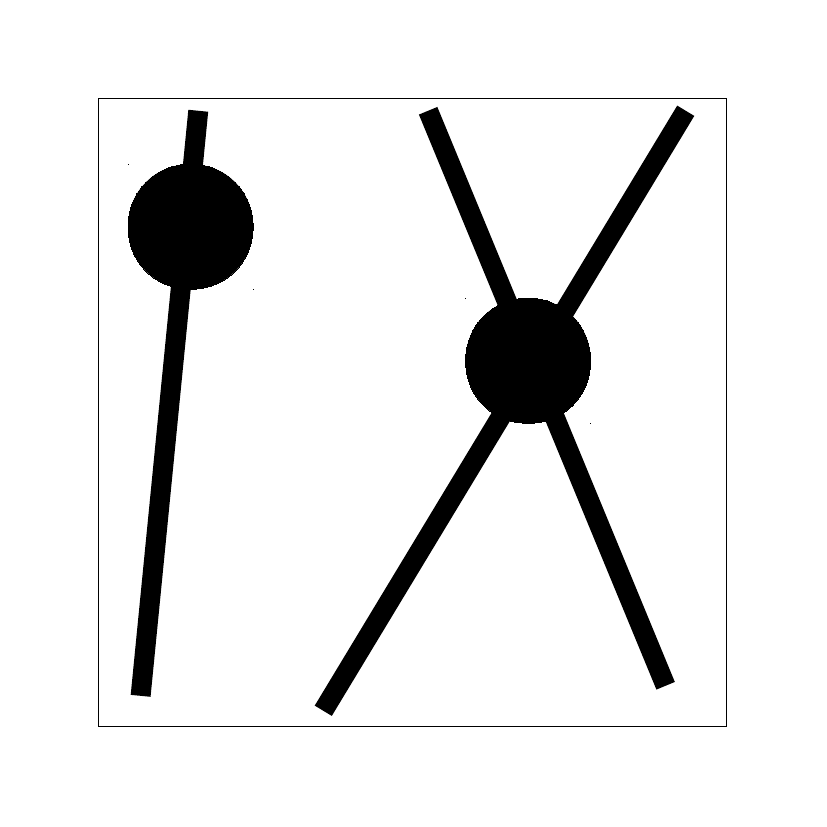} 
&\includegraphics[width=\LTW\textwidth]{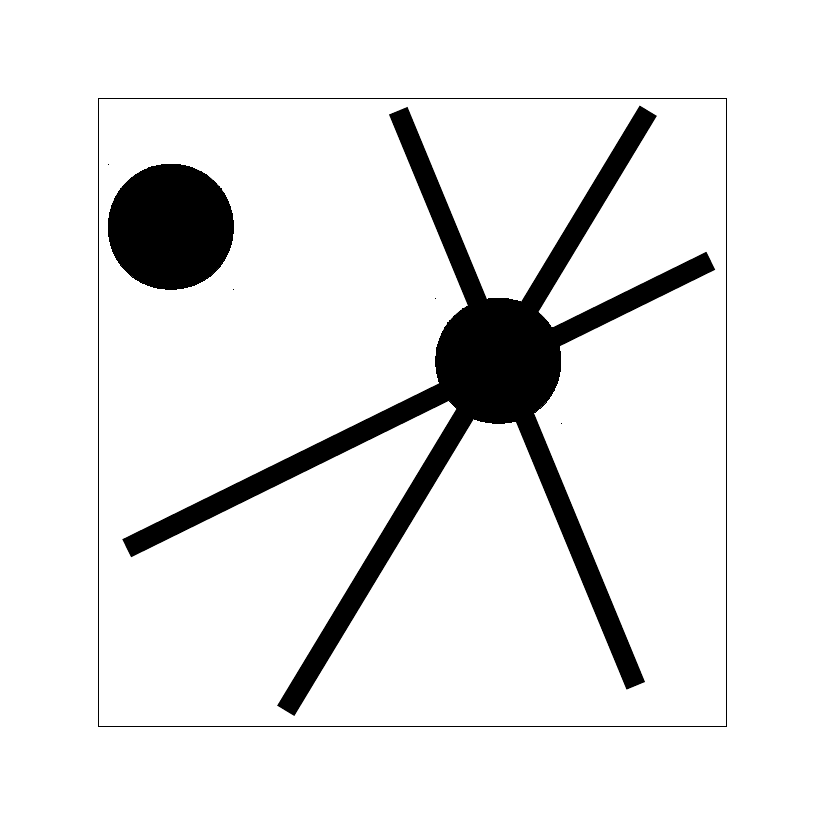} 
&\includegraphics[width=\LTW\textwidth]{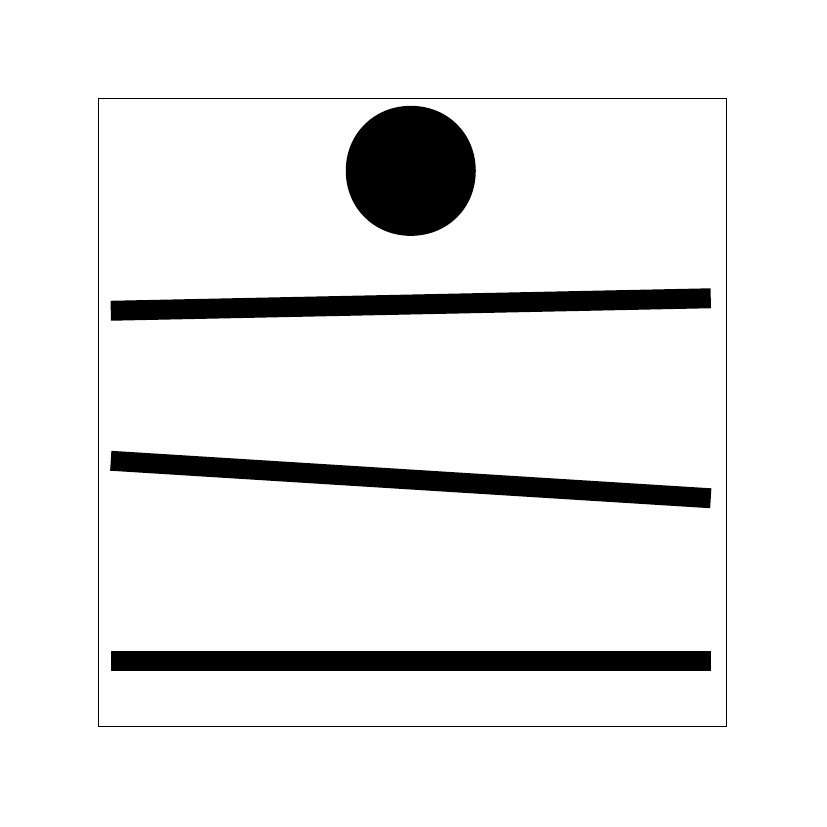} 
&\includegraphics[width=\LTW\textwidth]{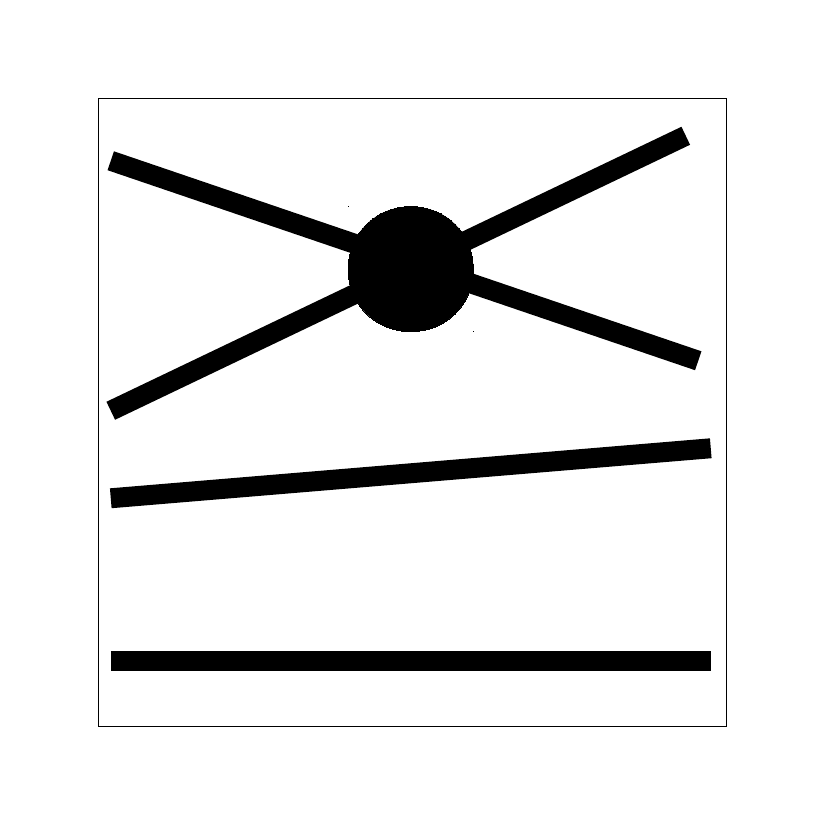} 
&\includegraphics[width=\LTW\textwidth]{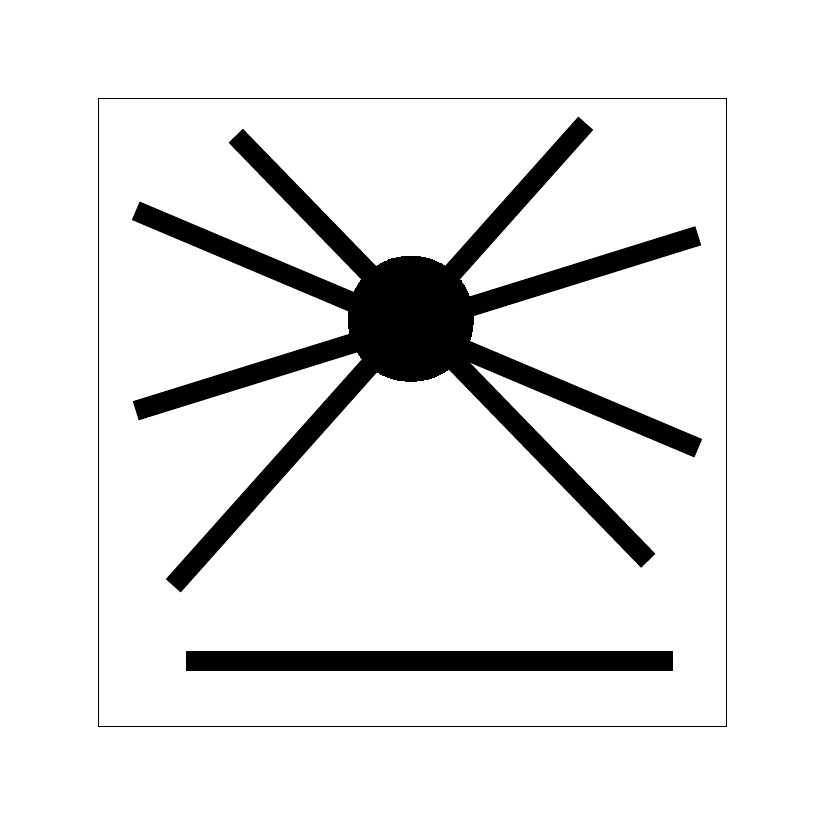} 
&\includegraphics[width=\LTW\textwidth]{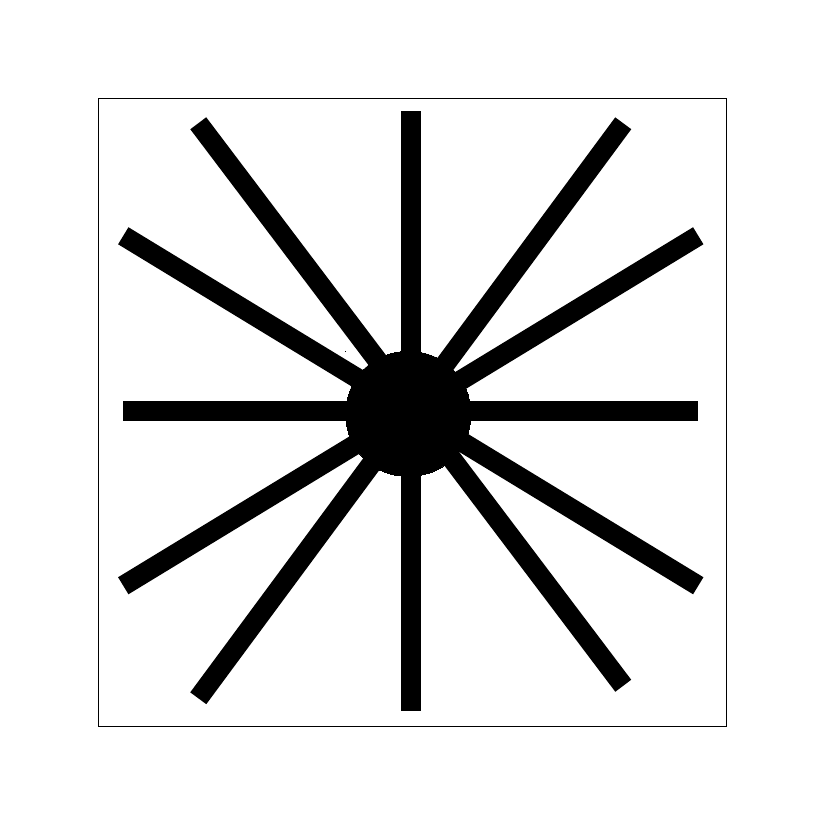} 
&\includegraphics[width=\LTW\textwidth]{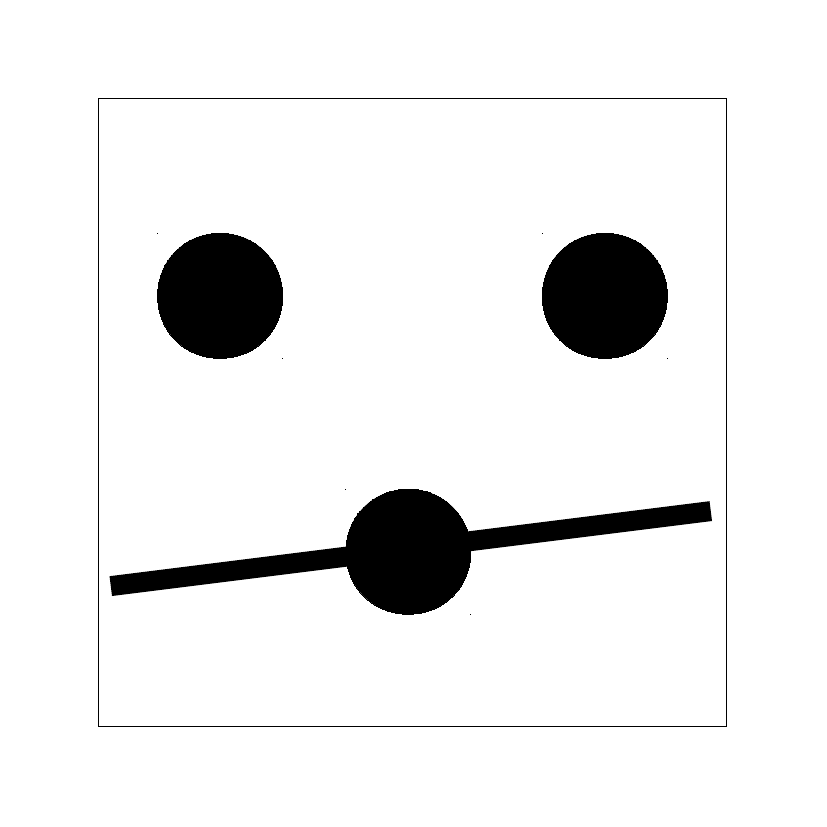} 
&\includegraphics[width=\LTW\textwidth]{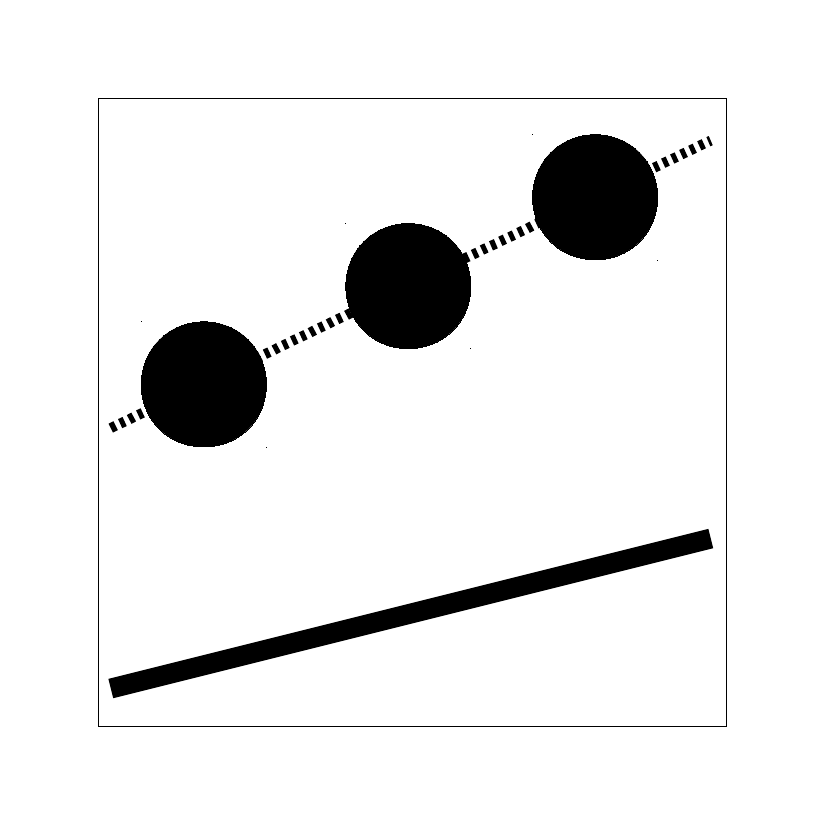} 
&\includegraphics[width=\LTW\textwidth]{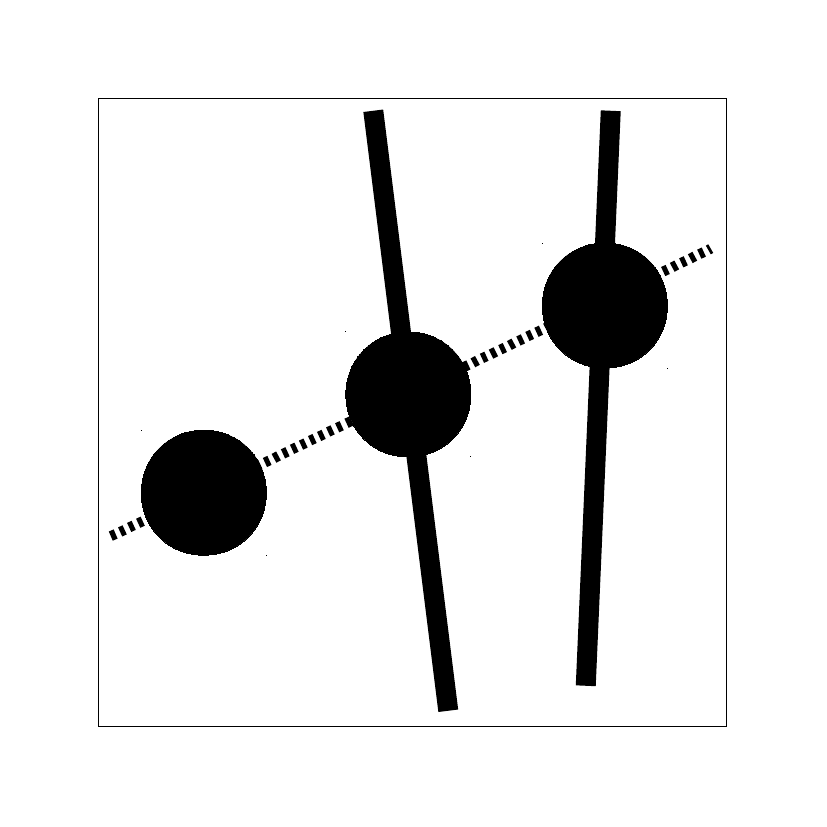} 
\\
Minimal &Y&N&N&Y&Y&Y&Y&Y&N&N&Y&Y&Y \\
Degree  &$>180k^*$&&&$11296^*$&$26240^*$&$11008^*$&$3040^*$&$4512^*$&&&$1728^*$&$32^*$&$544^*$ \\ 
\hline \hline
$m$ views &4&3&3&3&3&3&3&3&3&3&3&3&3 \\
$p^\mathrm{f}p^\mathrm{d}l^\mathrm{f}l^\mathrm{a}_\alpha$&
$2102_2$&$1040_0$&$1032_2$&$1024_4$&$1016_6$&$1008_8$&$2021_1$&$2013_2$&$2013_3$&$2005_3$&$2005_4$&$2005_5$&$3010_0$
\\
\parbox[b]{1.1cm}{$(p,l,\ICa)$\\[\LTC\textwidth]} 
&\includegraphics[width=\LTW\textwidth]{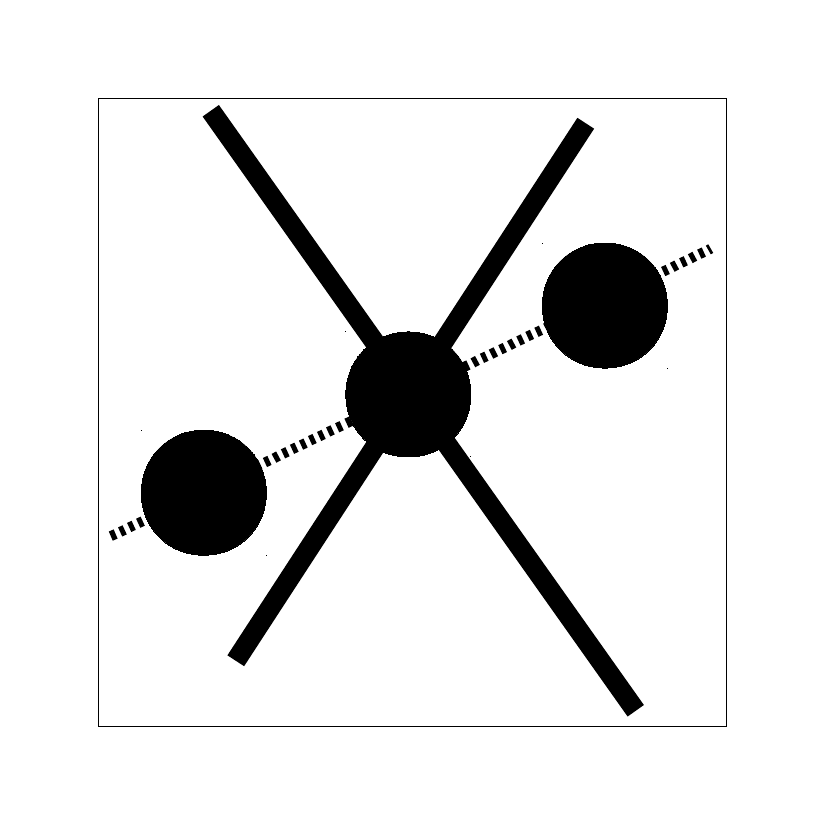} 
&\includegraphics[width=\LTW\textwidth]{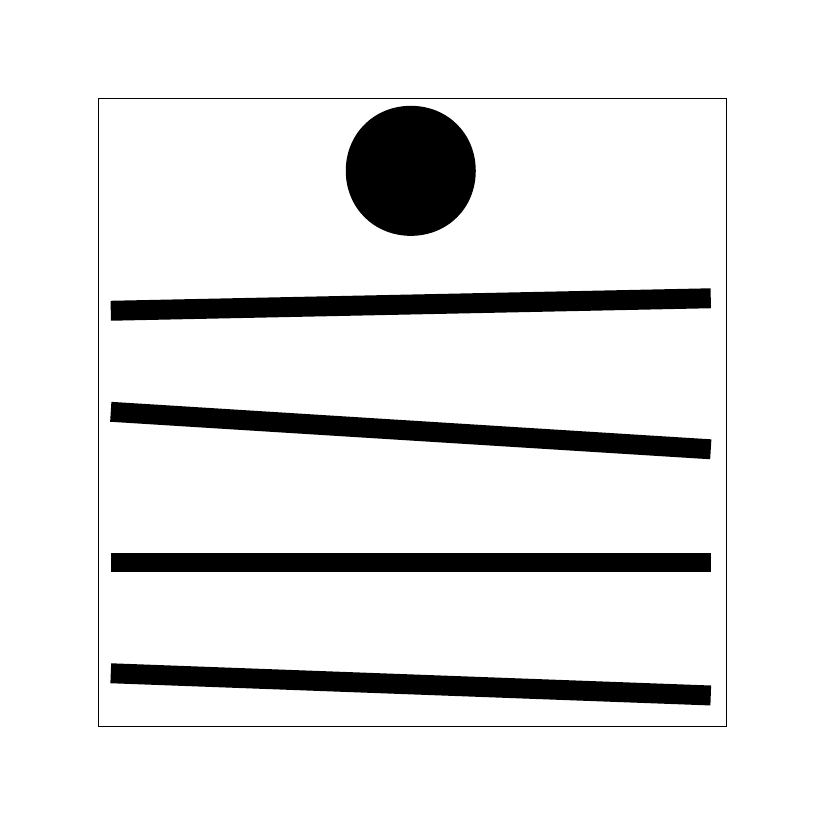} 
&\includegraphics[width=\LTW\textwidth]{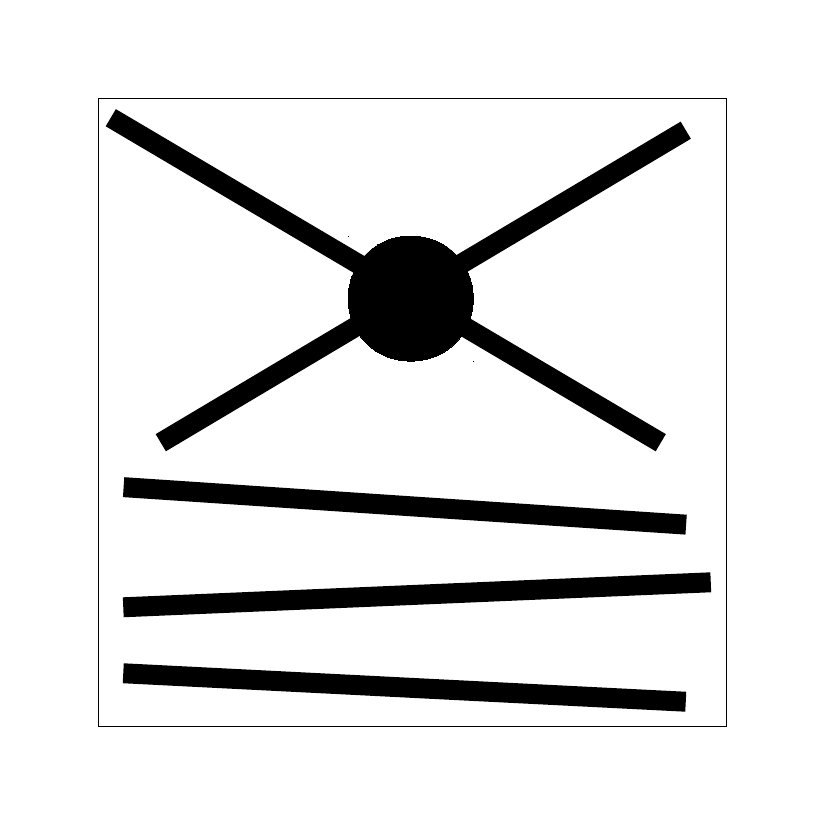} 
&\includegraphics[width=\LTW\textwidth]{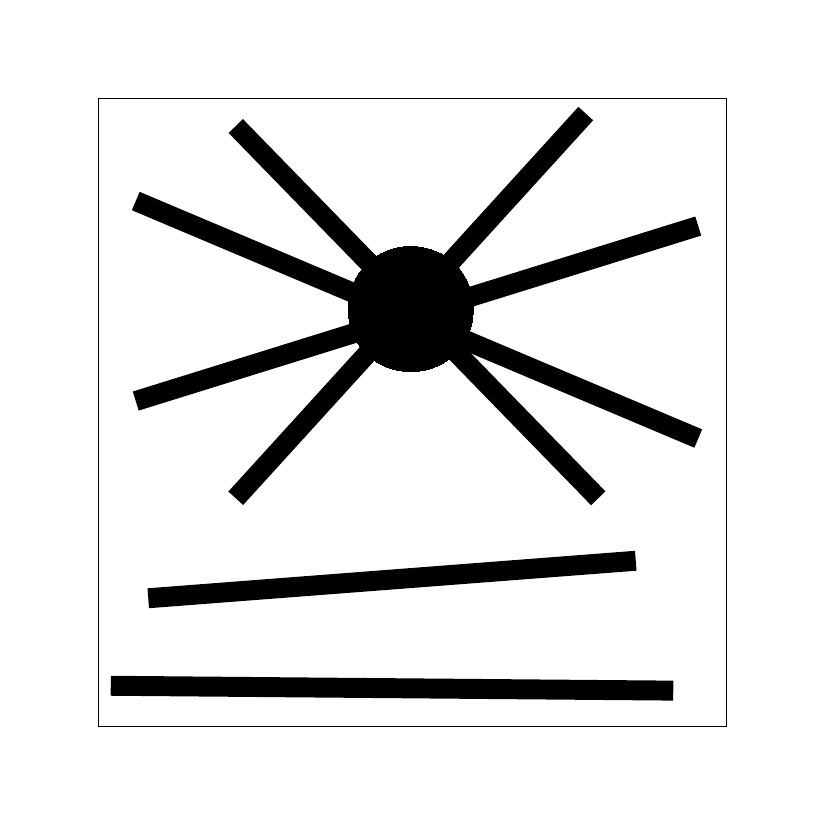} 
&\includegraphics[width=\LTW\textwidth]{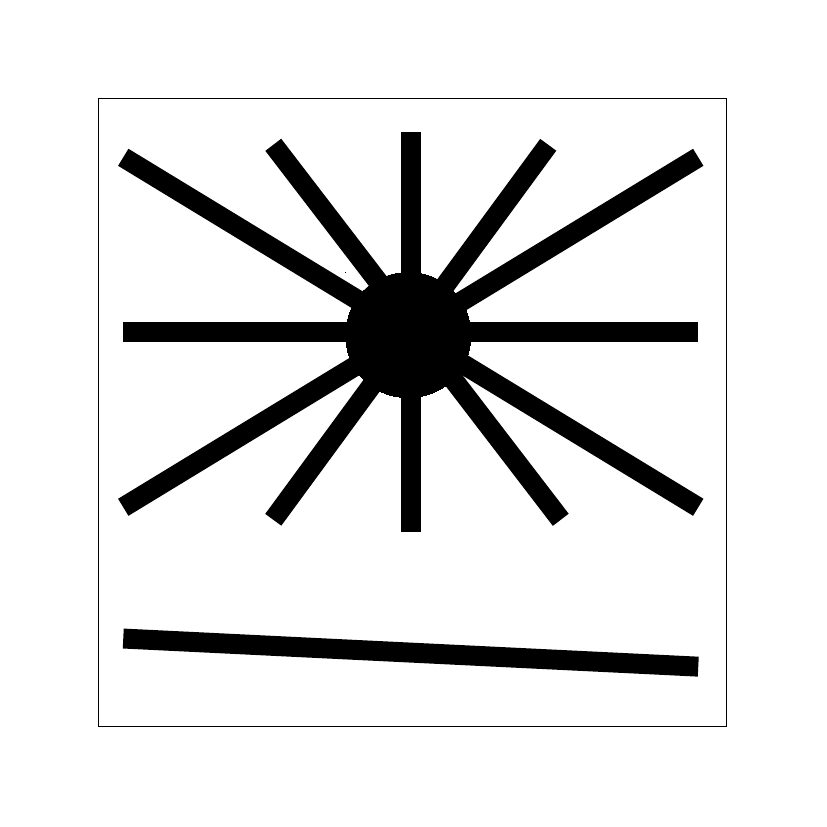} 
&\includegraphics[width=\LTW\textwidth]{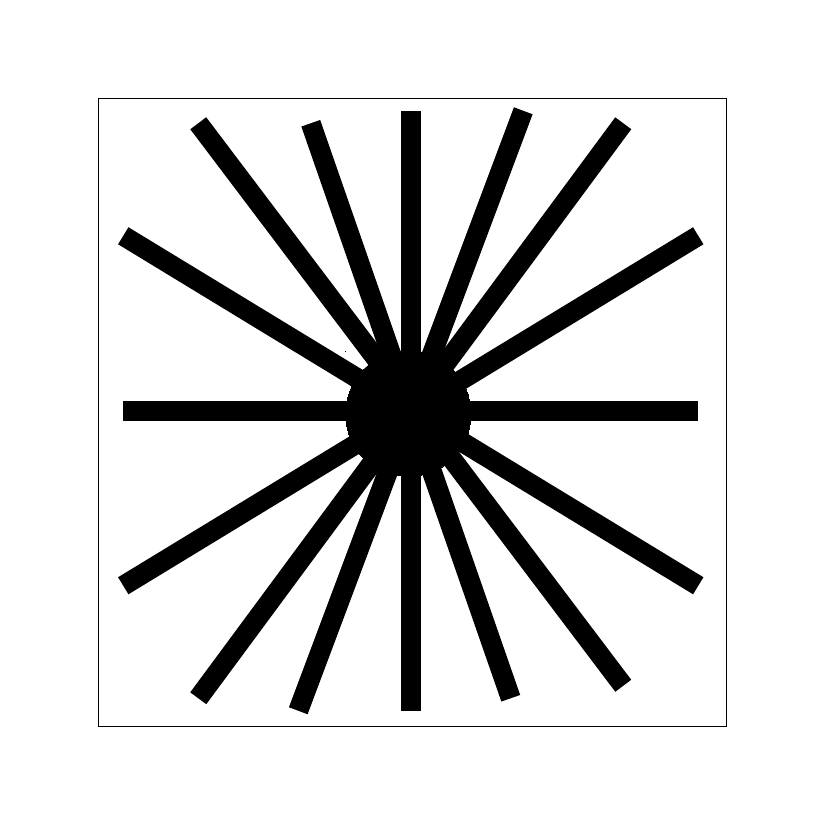} 
&\includegraphics[width=\LTW\textwidth]{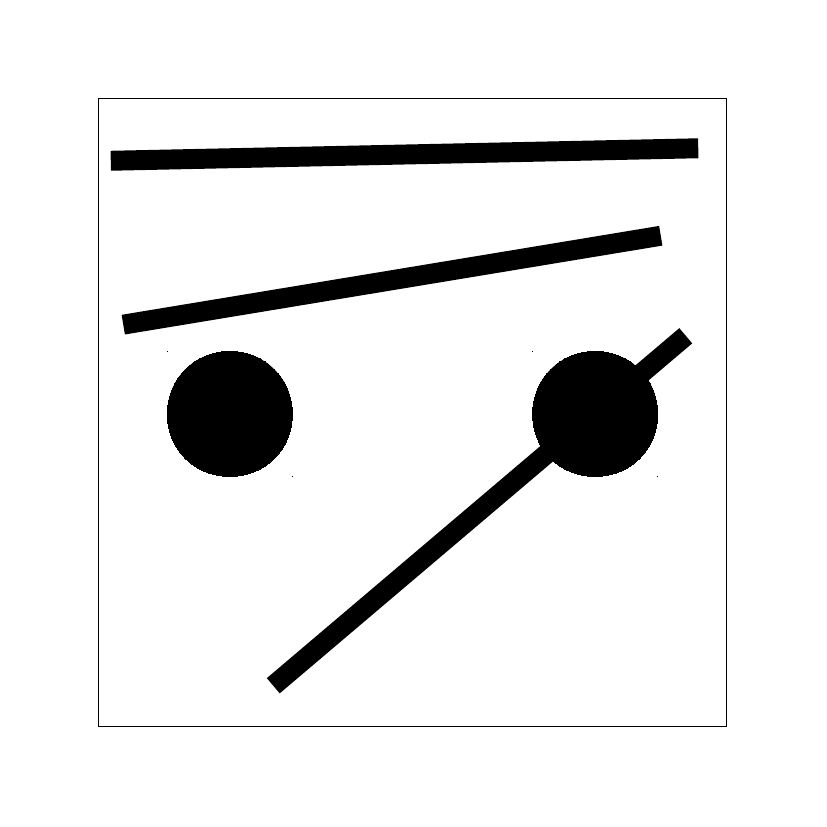} 
&\includegraphics[width=\LTW\textwidth]{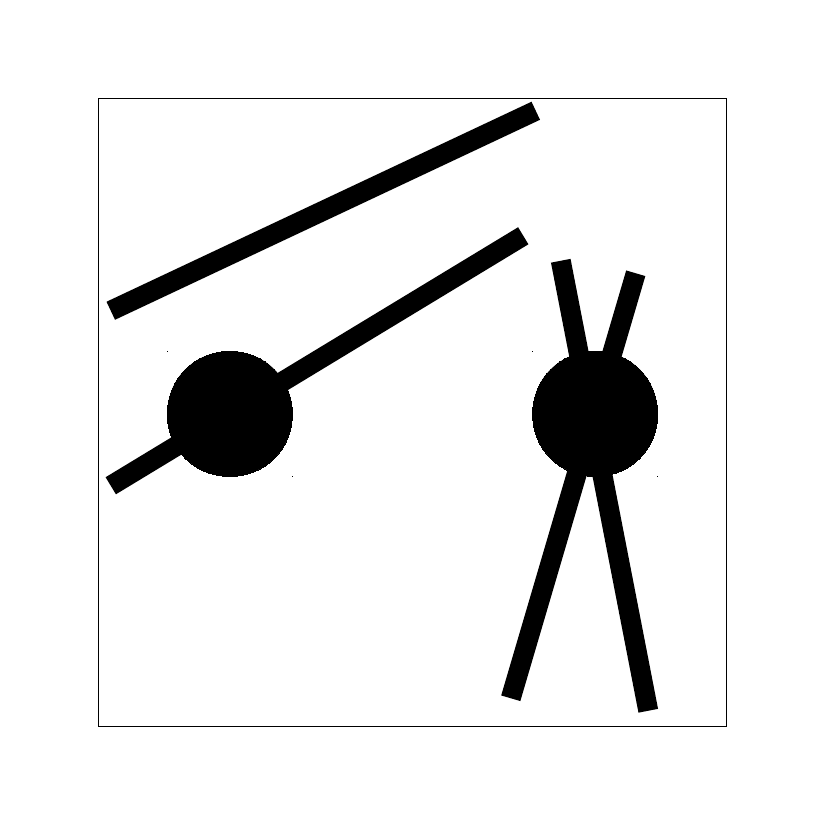} 
&\includegraphics[width=\LTW\textwidth]{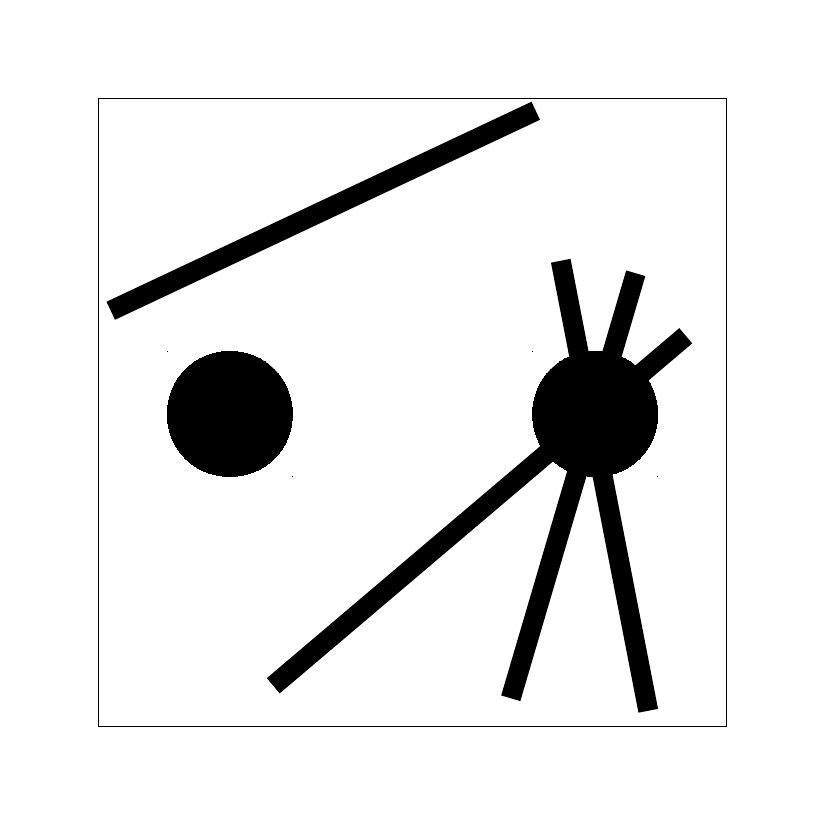} 
&\includegraphics[width=\LTW\textwidth]{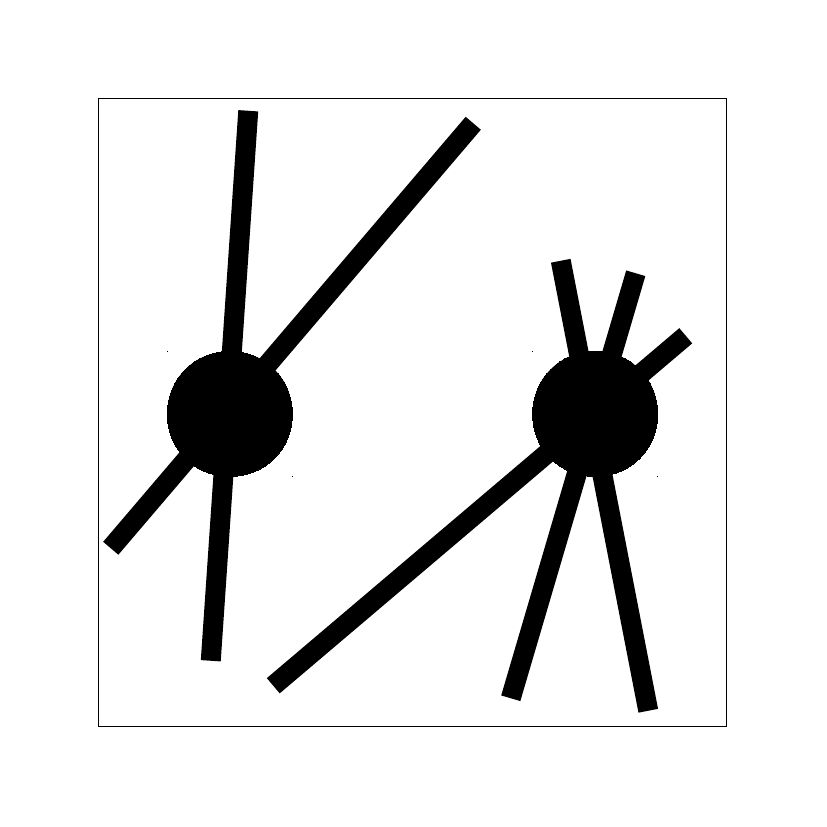} 
&\includegraphics[width=\LTW\textwidth]{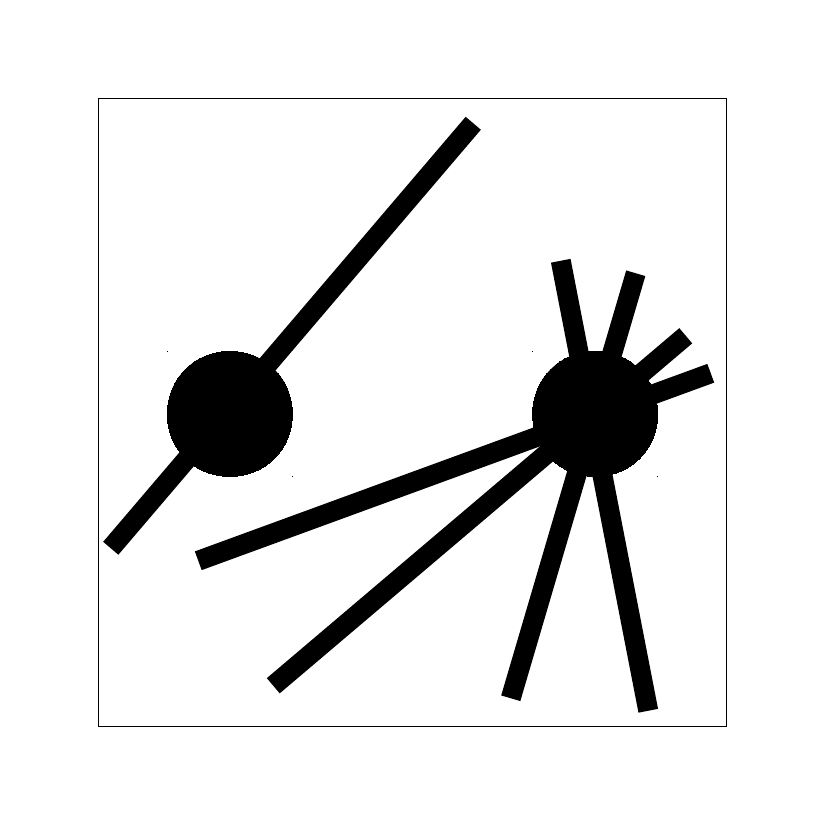} 
&\includegraphics[width=\LTW\textwidth]{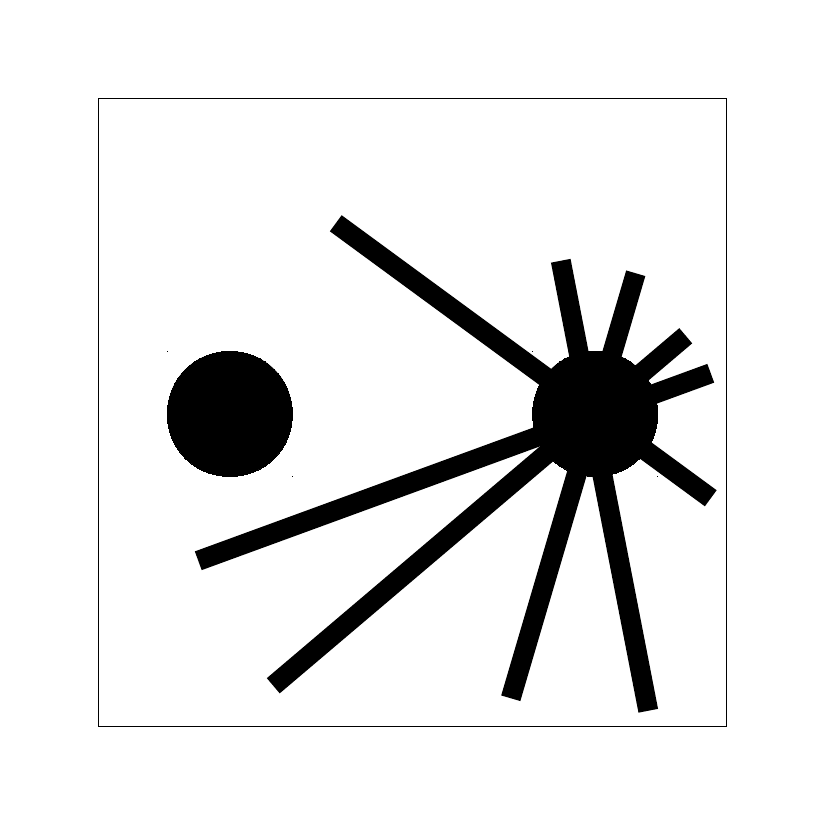} 
&\includegraphics[width=\LTW\textwidth]{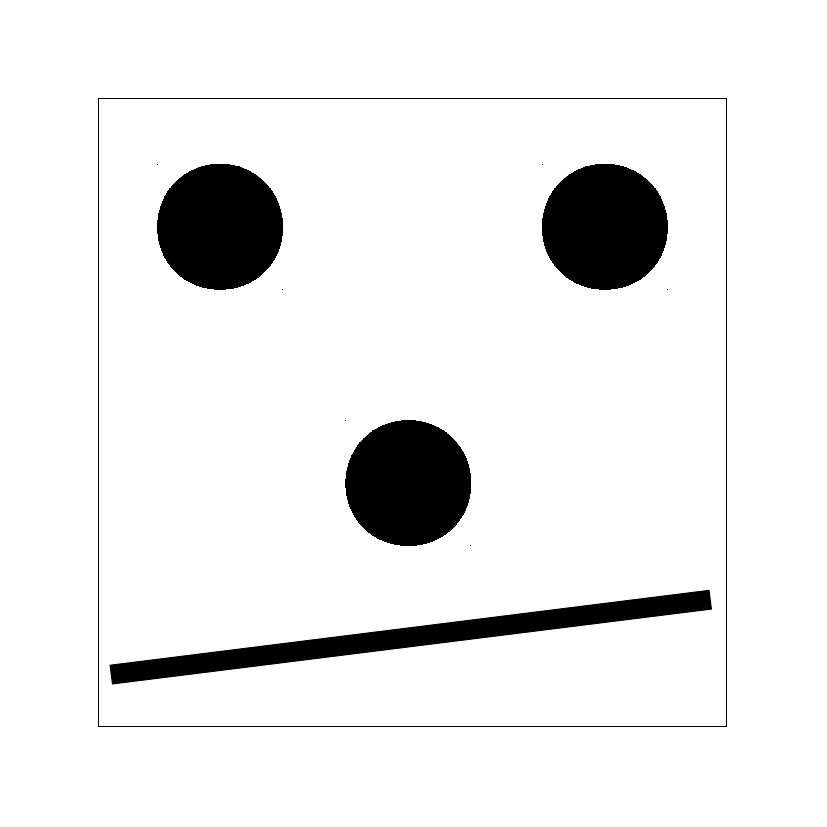} 
\\
Minimal &Y&Y  &Y  &Y&N&N&Y  &Y&Y&Y&Y&Y&Y \\
Degree  &$544^*$&360&552&480&&&264&432&328&480&240&64&216 \\ 
\hline \hline
$m$ views &3&3&3&3&3&3&3&3&2&2&2&2&2 \\
$p^\mathrm{f}p^\mathrm{d}l^\mathrm{f}l^\mathrm{a}_\alpha$&
$3002_1$&$3002_2$&$2111_1$&$2103_1$&$2103_2$&$2103_3$&$3100_0$&$2201_1$&$5000_2$&$4100_3$&$3200_3$&$3200_4$&$2300_5$
\\
\parbox[b]{1.1cm}{$(p,l,\ICa)$\\[\LTC\textwidth]} 
&\includegraphics[width=\LTW\textwidth]{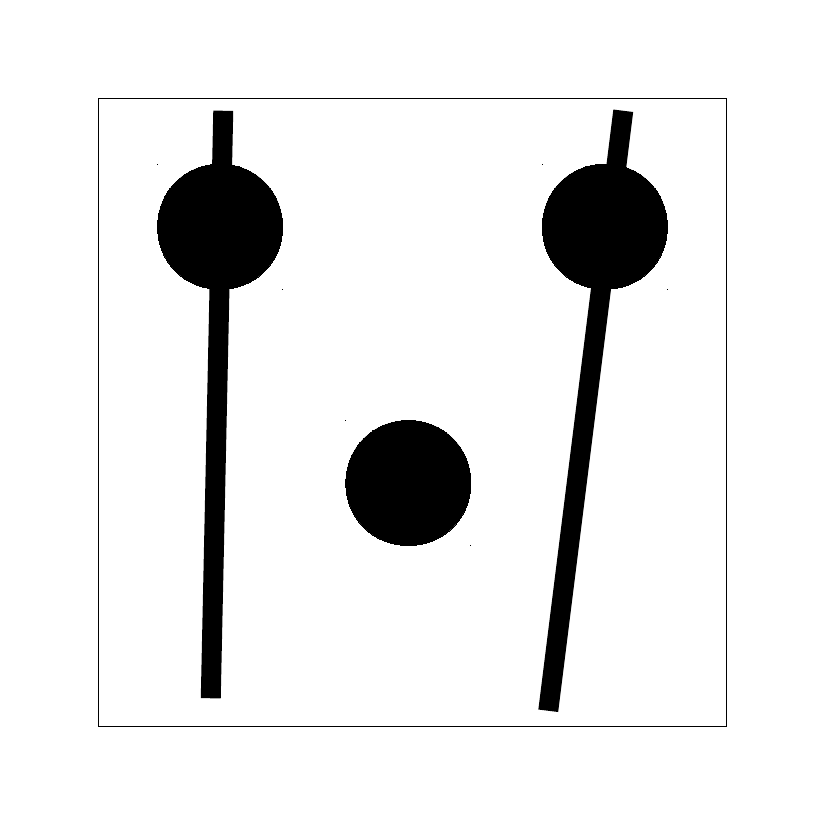} 
&\includegraphics[width=\LTW\textwidth]{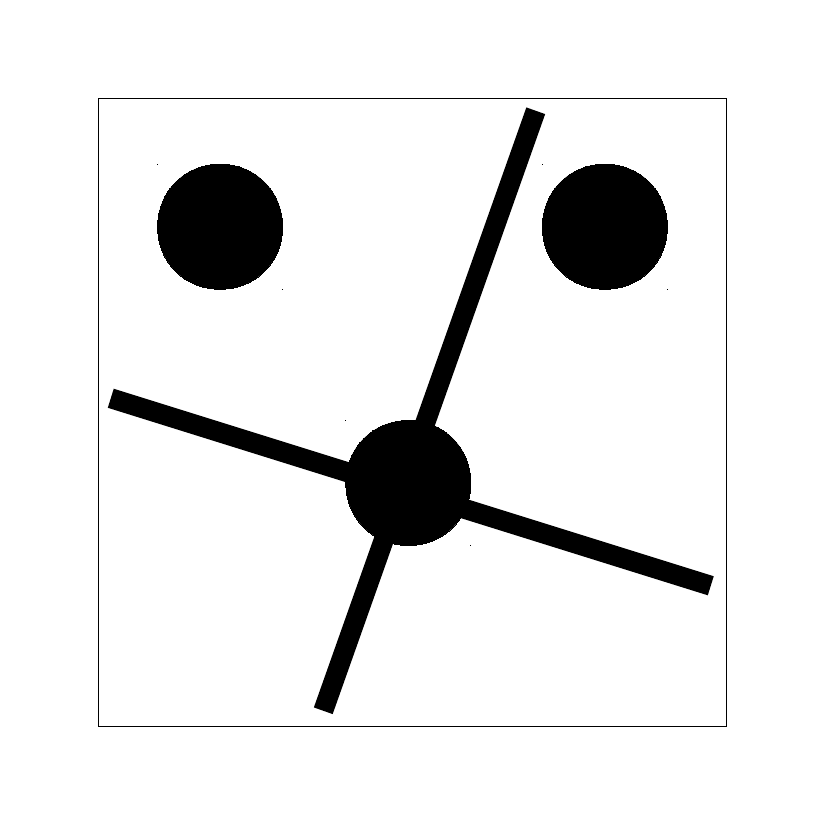} 
&\includegraphics[width=\LTW\textwidth]{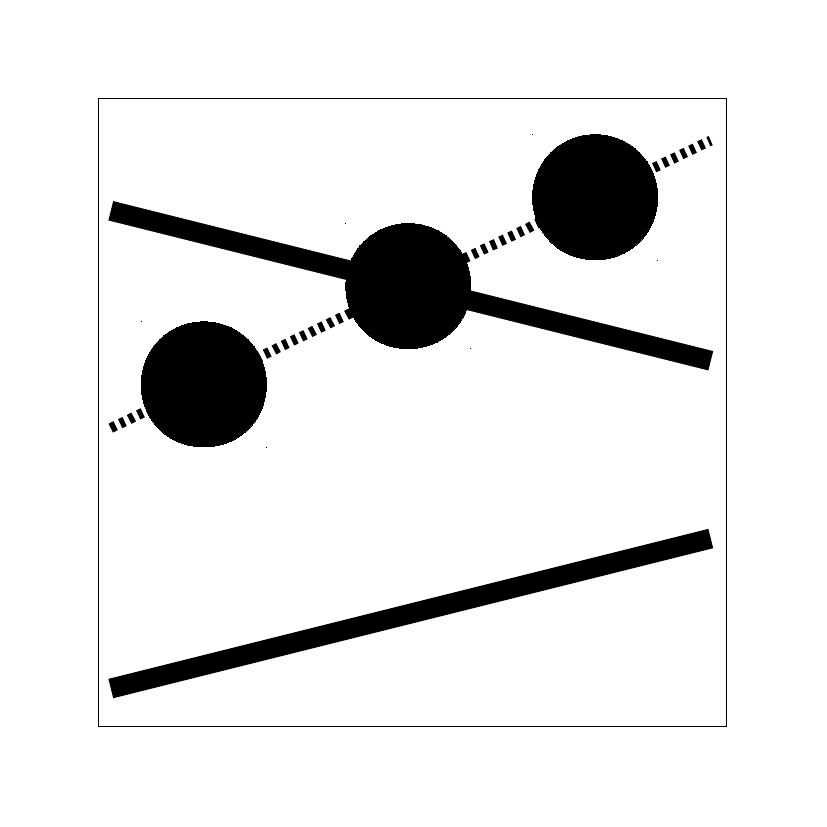} 
&\includegraphics[width=\LTW\textwidth]{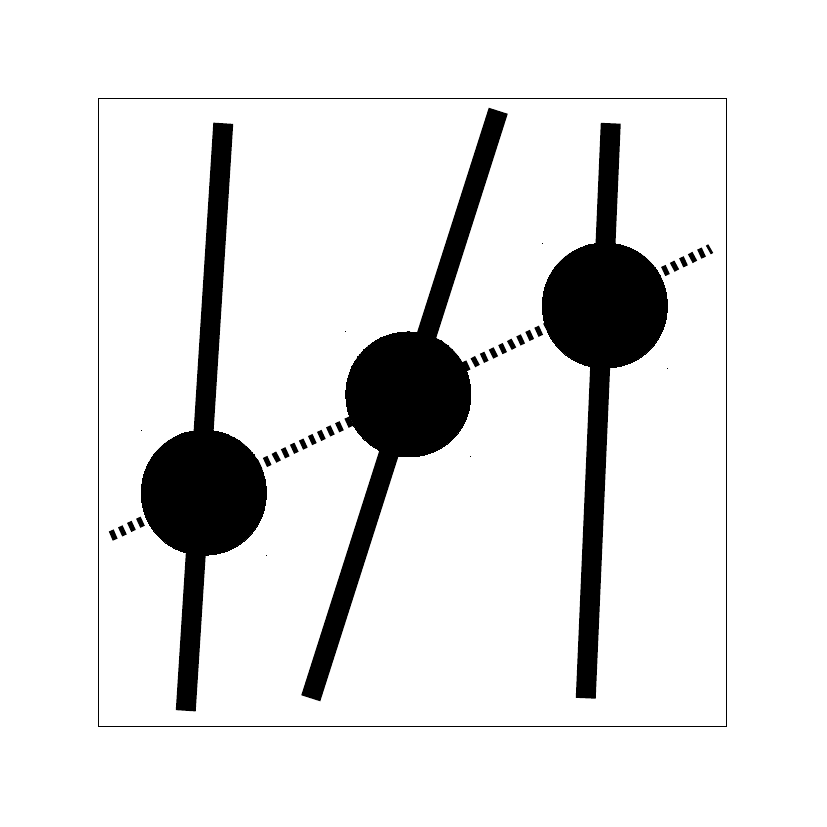} 
&\includegraphics[width=\LTW\textwidth]{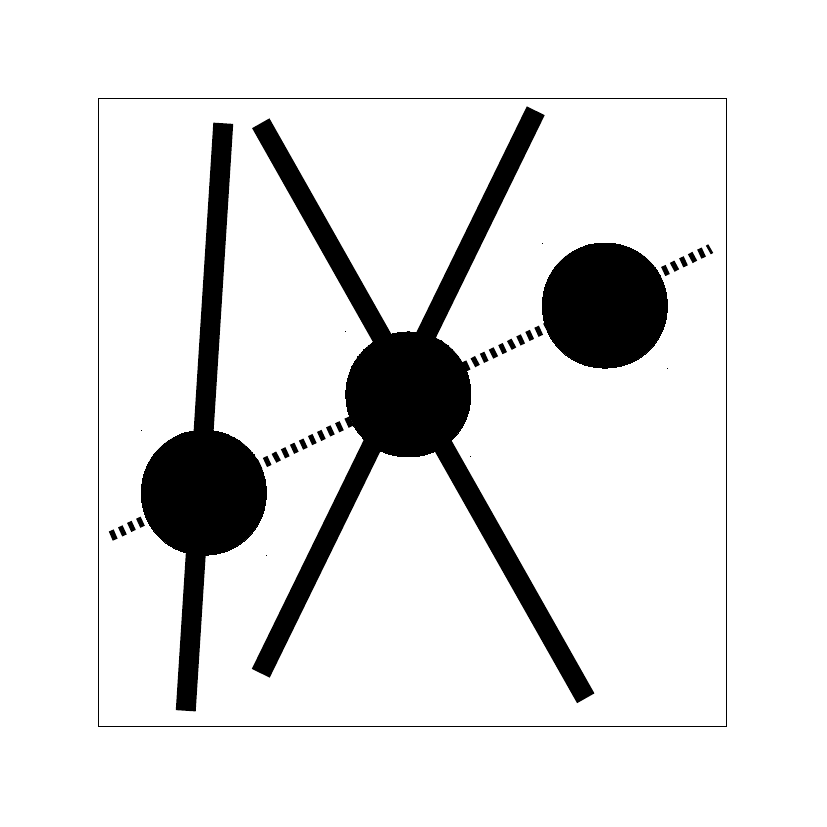} 
&\includegraphics[width=\LTW\textwidth]{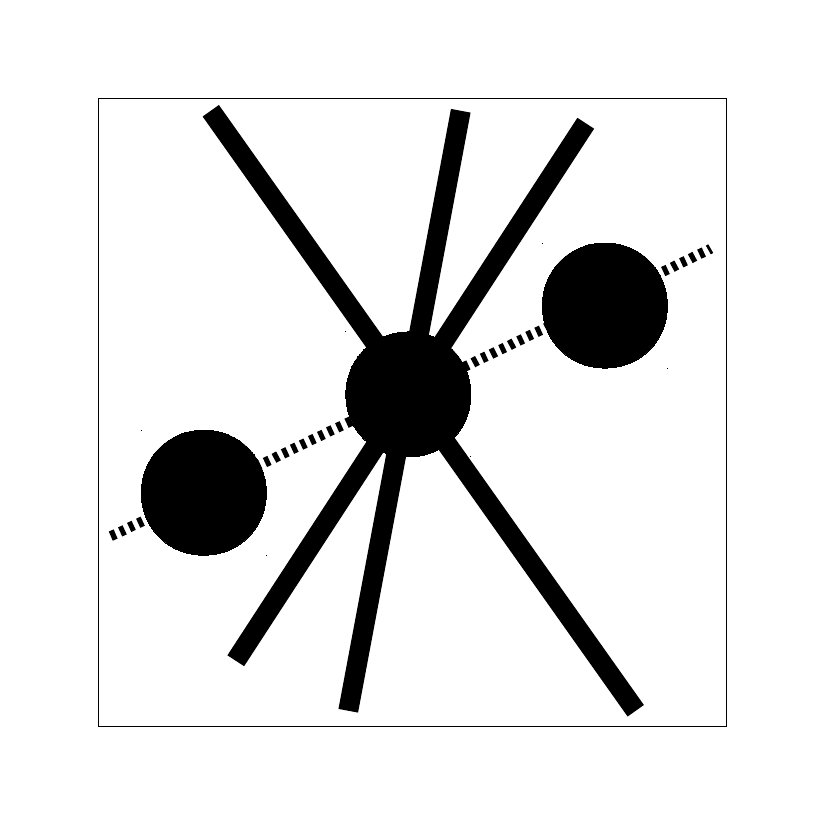} 
&\includegraphics[width=\LTW\textwidth]{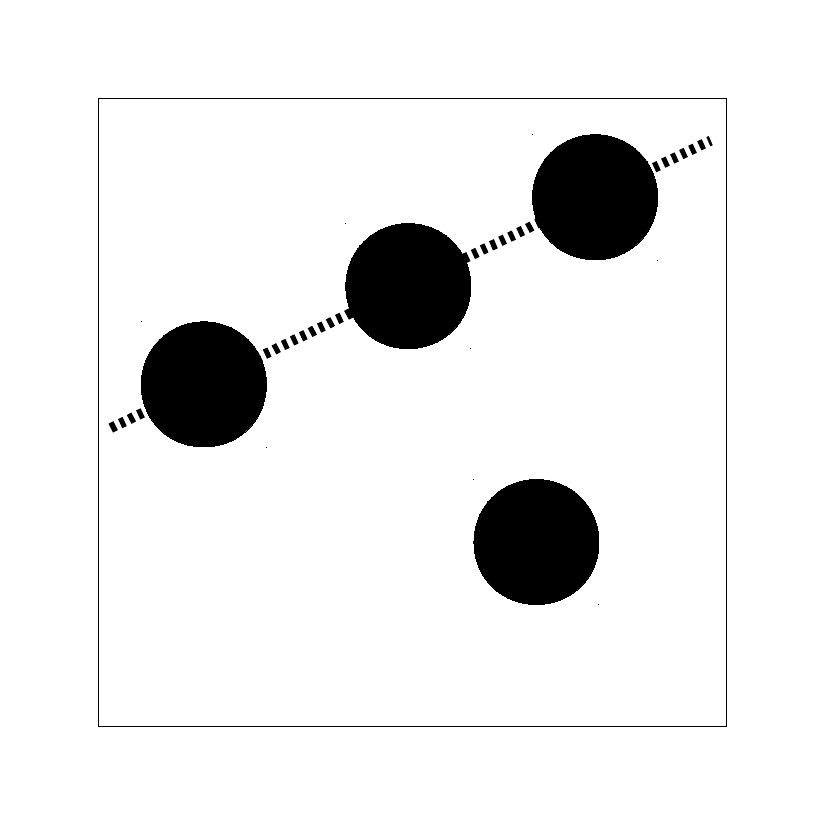} 
&\includegraphics[width=\LTW\textwidth]{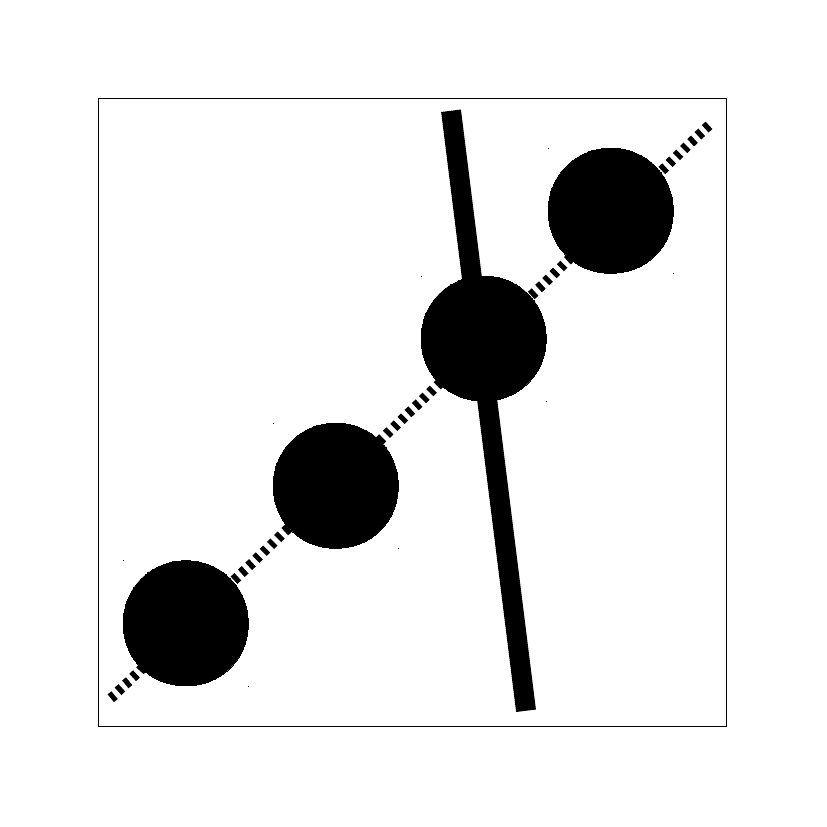} 
&\includegraphics[width=\LTW\textwidth]{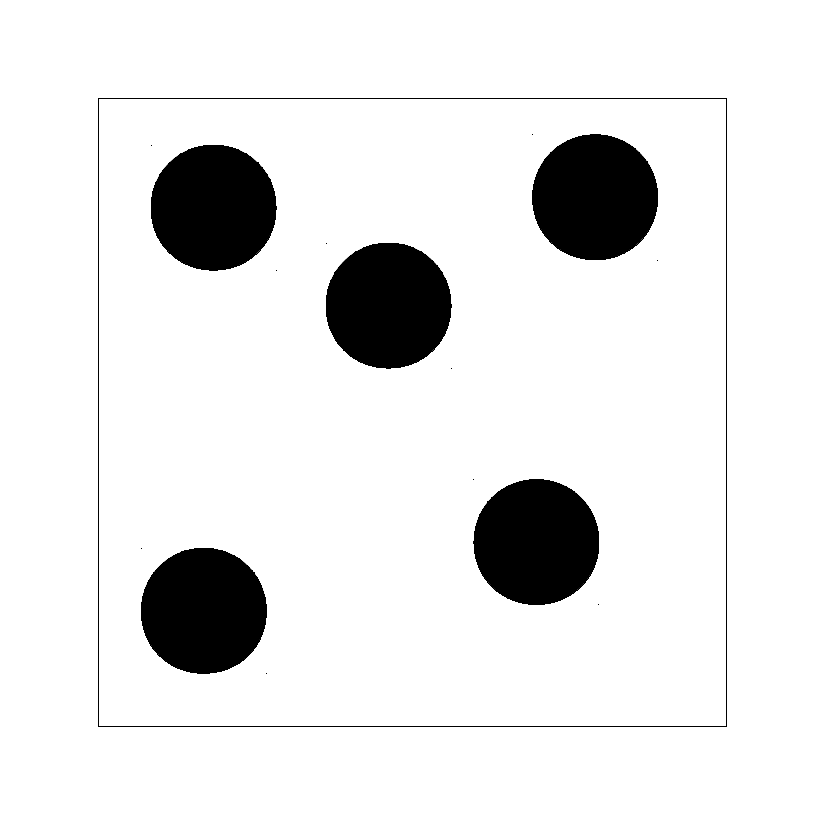}
&\includegraphics[width=\LTW\textwidth]{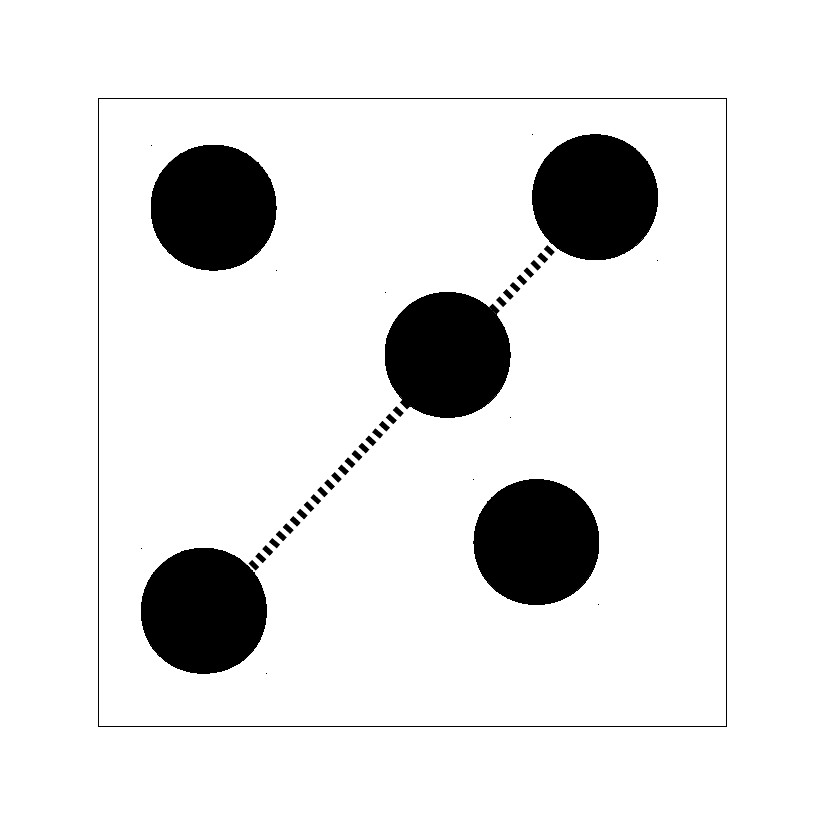}
&\includegraphics[width=\LTW\textwidth]{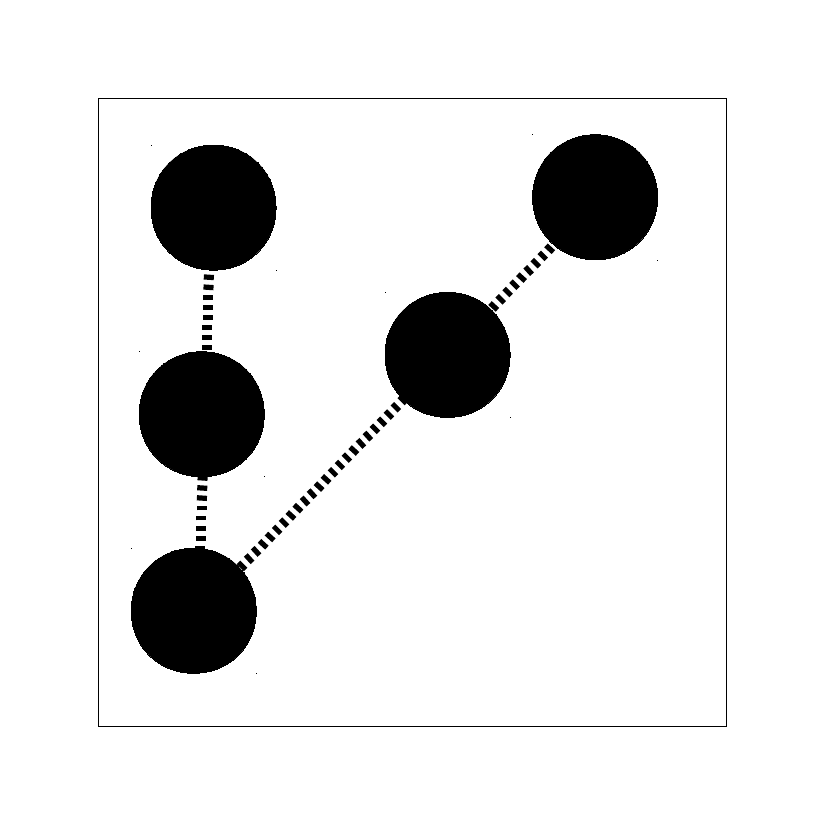}
&\includegraphics[width=\LTW\textwidth]{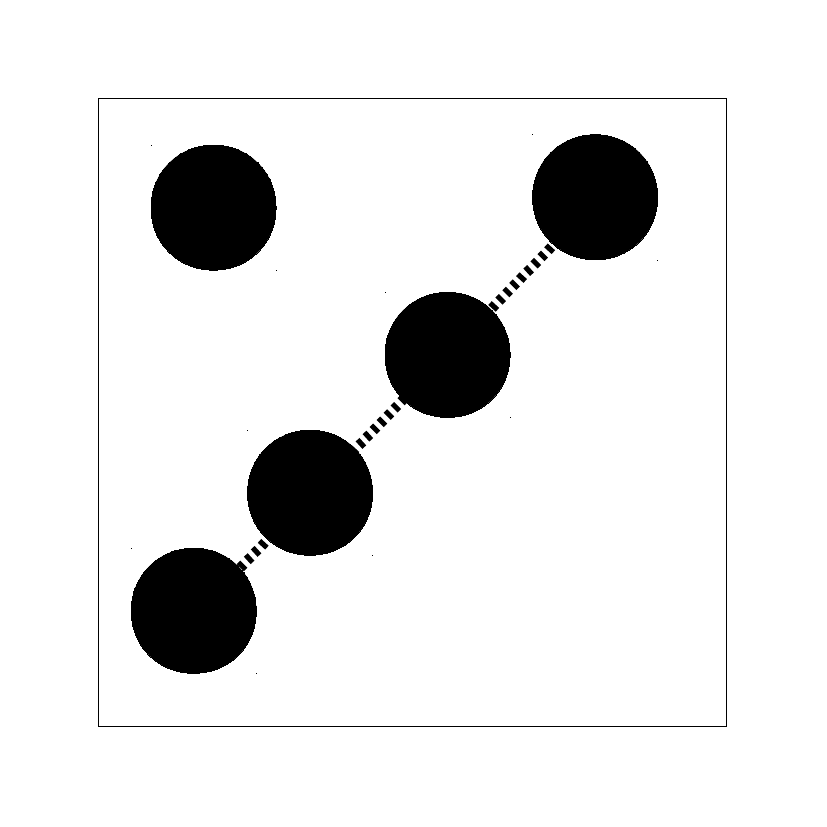}
&\includegraphics[width=\LTW\textwidth]{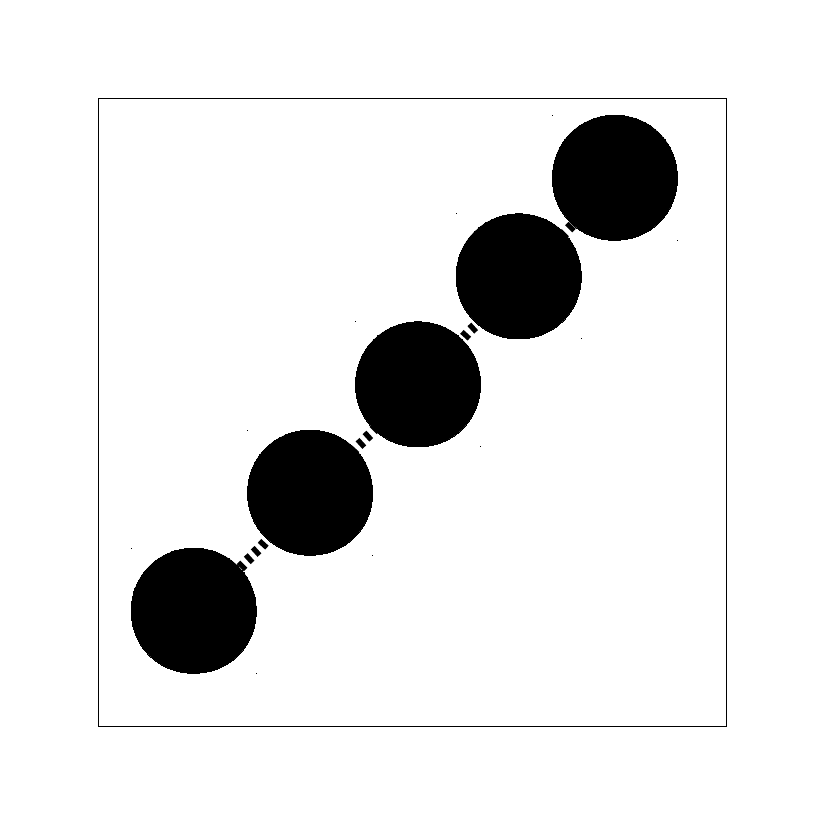}
\\
Minimal &Y&Y&Y&Y&Y&Y&Y&N&Y&Y&Y&N&N\\
Degree  &312&224&40&144&144&144&64&&20&16&12&& \\
\hline
\hline
\end{tabular}
\end{table*}
\addtolength{\tabcolsep}{\LTS} 
\section{Balanced Point-Line Problems}\label{sec:balanced-problems}
\noindent To understand balanced point-line problems we need to derive formulas for the dimensions of the varieties $\PplI$, $\Cm$ and $\YplIm$.
As $\SO(3)$ is three-dimensional, and we set the first camera to $[I|0]$ and one parameter in the second camera to $1$, the parameter space of camera configurations for $m\ge 2$ has dimension
$
    \dim (\Cm) = 6m-7.
$

Let us now consider a generic point-line arrangement in $\PplI$.
Some of its points may be dependent on other points, in the sense that such a dependent point lies on a line spanned by two other points. In any arrangement of points in $3$-space, each minimal set of independent points has the same cardinality. For our arrangement of $p$ points we denote this cardinality by $p^\mathrm{f}$ (the upper index $\mathrm{f}$ stands for \emph{free}). We write $p^\mathrm{d} = p - p^\mathrm{f}$ for the number of dependent points. Each free point is defined by three parameters. A dependent point $X$ is only defined by \emph{one} further parameter after the two points, which span the line containing $X$, are defined. In total, the $p$ points in our arrangement are defined by $3\,p^\mathrm{f} + p^\mathrm{d}$ parameters. Each of the $l$ lines in our arrangement is either incident to zero, one or at least two points.  We refer to lines which are incident to no points as \emph{free lines}. We denote the number of free lines by $l^\mathrm{f}$. As the Grassmannian $\GG_{1,3}$ of lines is four-dimensional, each free line is defined by four parameters. A line which is incident to a fixed point is defined by only two parameters. We denote the number of lines which are incident to exactly one point by $l^\mathrm{a}$ (the upper index $\mathrm{a}$ stands for \emph{adjacent}). Finally, each of the remaining $l - l^\mathrm{f} - l^\mathrm{a}$ lines is incident to at least two points and thus already uniquely determined by the two points. 
Hence, we have derived
\begin{align}
\label{eq:dimArrangements}
    \dim (\PplI) = 3\,p^\mathrm{f} + p^\mathrm{d} + 4\,l^\mathrm{f} + 2\,l^\mathrm{a}.
\end{align}
In particular, we see that we might as well assume that there is no line passing through two or more points, as such lines do not contribute to our parameter count.

We derive the dimension of the image variety $\YplIm$ similarly. Since we assume all camera positions to be sufficiently generic, each camera views exactly $p^\mathrm{f}$ independent points, $p^\mathrm{d}$ dependent points, $l^\mathrm{f}$ free lines and $l^\mathrm{a}$ lines which are incident to exactly one of the points. Each independent point is defined by two parameters, whereas each dependent point is defined by a single parameter. A free line is defined by two parameters. A line which is incident to a fixed point is defined by a single parameter. All in all, we have that
\begin{align}
\label{eq:dimArrangementsCamera}
     \dim  (\YplIm) = m\, ( 2\,p^\mathrm{f} + p^\mathrm{d} + 2\,l^\mathrm{f} + l^\mathrm{a} ).
\end{align}
Note that there is no balanced point-line problem for a single camera.
For $m > 1$ cameras, combining $\dim (\Cm) = 6m-7$ with ~\eqref{eq:dimArrangements} and~\eqref{eq:dimArrangementsCamera} yields that a point-line problem is balanced if and only if
\begin{align*}
     3\,p^\mathrm{f} + p^\mathrm{d} + 4\,l^\mathrm{f} + 2\,l^\mathrm{a} + 6\,m-7
     = m \left( 2\,p^\mathrm{f} + p^\mathrm{d} + 2\,l^\mathrm{f} + l^\mathrm{a} \right).
\end{align*}
This is equivalent to 
\begin{align}
    \label{eq:balancedCondition}
 \hspace*{-2mm}    6m\!-\!7\!=\!(2m\!-\!3)p^\mathrm{f}\!+\!(m\!-\!1)p^\mathrm{d}\!+\!2(m\!-\!2)l^\mathrm{f}\!+\!(m\!-\!2)l^\mathrm{a}.
\end{align}

\begin{lemma}
\label{lem:fivePoints}
Every balanced point-line problem with at least five points has exactly two cameras.
\end{lemma}
\begin{proof}
Suppose $(\PLP)$ is a balanced point-line problem with $m > 1$ cameras and at least five points, \ie $p^\mathrm{f} + p^\mathrm{d} \geq  5$.
In this case, the equality~\eqref{eq:balancedCondition} implies
\begin{align*}
    6m\!-\!7 \geq (2m\!-\!3)p^\mathrm{f}\!+\!(m\!-\!1) (5\!-\!p^\mathrm{f})\!=\!(p^\mathrm{f}\!+\!5)m\!-\!(2p^\mathrm{f}\!+\!5),
\end{align*}
which is equivalent to 
\begin{align}
\label{eq:fivePointsLemma}
    2(p^\mathrm{f}-1) \geq (p^\mathrm{f}-1)m.
\end{align}
Among the five or more points at least two have to be (by definition) independent, \ie $p^\mathrm{f} > 1$. So~\eqref{eq:fivePointsLemma} yields $m \leq 2$.
\end{proof}
\begin{theorem}\label{thm:less-than-seven}
There is no balanced point-line problem with seven or more cameras. 
\end{theorem}
 \begin{proof}
Let $(\PLP)$ be a balanced point-line problem with $m \geq 7$ cameras. By equality~\eqref{eq:balancedCondition}, we have
\begin{align}
\label{eq:balancedConditionMod}
    5 \equiv p^\mathrm{f} + p^\mathrm{d} \mod (m-2).
\end{align}
This implies $p^\mathrm{f} + p^\mathrm{d} \geq 5$ if $m \geq 8$, which contradicts Lemma~\ref{lem:fivePoints}, and thus we have only one remaining case to check: \mbox{$m=7$}. From~\eqref{eq:balancedConditionMod} and Lemma~\ref{lem:fivePoints}, 
we have $p^\mathrm{f} + p^\mathrm{d} = 0$ in the case of seven cameras. It means that there are no points, and thus there cannot be lines which are incident to points. So we have $p^\mathrm{f} = 0$, $p^\mathrm{d} = 0$, $l^\mathrm{a} = 0$, and~\eqref{eq:balancedCondition} reduces to 
$35 = 10 l^\mathrm{f}$, which is clearly not possible. So there are no balanced point-line problems with seven or more cameras.
\end{proof}
\begin{theorem}\label{thm:table}
There are $34$ balanced point-line problems with $3$, $4$, $5$ or $6$ cameras. They are all listed in Tab.~\ref{tab:balanced}.
\end{theorem}
\begin{proof}
We consider the different cases for $3 \leq m \leq 6$ and reason by cases.

\noindent $\bullet\  m=6$:
Due to~\eqref{eq:balancedConditionMod} and Lemma~\ref{lem:fivePoints}, every balanced point-line problem with six cameras must have exactly one point. So we have $p^\mathrm{f} = 1$, $p^\mathrm{d} = 0$, and~\eqref{eq:balancedCondition} reduces to $5 = 2 l^\mathrm{f} +  l^\mathrm{a}$. This gives us three possibilities: $(l^\mathrm{f}, l^\mathrm{a}) \in \lbrace (2,1), (1,3), (0,5) \rbrace$ (see first row of Tab.~\ref{tab:balanced}).

\noindent $\bullet\  m=5$:
Due to~\eqref{eq:balancedConditionMod} and Lemma~\ref{lem:fivePoints}, every balanced point-line problem with five cameras must have exactly two points. So we have $p^\mathrm{f} = 2$, $p^\mathrm{d} = 0$, and~\eqref{eq:balancedCondition} reduces to $3 = 2 l^\mathrm{f} +  l^\mathrm{a}$. This gives us two possibilities: $(l^\mathrm{f}, l^\mathrm{a}) \in \lbrace (1,1), (0,3) \rbrace$, which yield three point-line problems (see the first row of Tab.~\ref{tab:balanced}).

\noindent $\bullet\  m=4$:
Due to~\eqref{eq:balancedConditionMod} and Lemma~\ref{lem:fivePoints}, every balanced point-line problem with four cameras must have either one point or three points. Let us first consider the case of a single point. Here we have $p^\mathrm{f} = 1$, $p^\mathrm{d} = 0$, and~\eqref{eq:balancedCondition} reduces to $6 = 2 l^\mathrm{f} +  l^\mathrm{a}$. This gives us four possibilities: $(l^\mathrm{f}, l^\mathrm{a}) \in \lbrace (3,0), (2,2), (1,4), (0,6) \rbrace$ (see first row of Tab.~\ref{tab:balanced}). Secondly, we consider balanced point-line problems with four cameras and three points.
If all three points are independent, \eqref{eq:balancedCondition} reduces to $1=2 l^\mathrm{f} + l^\mathrm{a}$, which has a single solution: $(l^\mathrm{f}, l^\mathrm{a}) = (0,1)$. If not all three points are independent,  we have $p^\mathrm{f} = 2$, $p^\mathrm{d} = 1$, and~\eqref{eq:balancedCondition} reduces to $2 =  2 l^\mathrm{f} +  l^\mathrm{a}$. This gives us two possibilities: $(l^\mathrm{f}, l^\mathrm{a}) \in \lbrace (1,0), (0,2) \rbrace$, which yield three point-line problems (see the first two rows of Tab.~\ref{tab:balanced} for all four point-line problems with four cameras and three points).

\noindent $\bullet\  m=3$: 
We first observe that each balanced point-line problem with three cameras must have at least one point. Otherwise we would have $p^\mathrm{f} = 0$, $p^\mathrm{d} = 0$ and $l^\mathrm{a} = 0$, so~\eqref{eq:balancedCondition} would reduce to $11 = 2l^\mathrm{f}$, which is impossible. Let us first consider the case of a single point.
Here we have $p^\mathrm{f} = 1$, $p^\mathrm{d} = 0$, and~\eqref{eq:balancedCondition} reduces to $8 = 2 l^\mathrm{f} +  l^\mathrm{a}$. This gives us five possibilities: $(l^\mathrm{f}, l^\mathrm{a}) \in \lbrace (4,0), (3,2), (2,4), (1,6), (0,8) \rbrace$ (see second row of Tab.~\ref{tab:balanced}). Secondly, in the case of two points, we have $p^\mathrm{f} = 2$, $p^\mathrm{d} = 0$, and~\eqref{eq:balancedCondition} reduces to $5 = 2 l^\mathrm{f} +  l^\mathrm{a}$. This gives us three possibilities: $(l^\mathrm{f}, l^\mathrm{a}) \in \lbrace (2,1), (1,3), (0,5) \rbrace$, which yield six point-line problems (see second row of Tab.~\ref{tab:balanced}). Thirdly, we consider the case of three points. If all three points are independent, \eqref{eq:balancedCondition} reduces to $2=2 l^\mathrm{f} + l^\mathrm{a}$. The two solutions $(l^\mathrm{f}, l^\mathrm{a}) \in \lbrace (1,0), (0,2) \rbrace$ yield three point line problems  (see last two rows of Tab.~\ref{tab:balanced}). If not all three points are independent,  we have $p^\mathrm{f} = 2$, $p^\mathrm{d} = 1$, and~\eqref{eq:balancedCondition} reduces to $3 =  2 l^\mathrm{f} +  l^\mathrm{a}$. The two solutions $(l^\mathrm{f}, l^\mathrm{a}) \in \lbrace (1,1), (0,3) \rbrace$ yield four point-line problems
(see last row of Tab.~\ref{tab:balanced}). Finally, we consider balanced point-line problems with three cameras and four points.
We see from~\eqref{eq:balancedCondition} that not all four points can be independent. Hence, we either have $p^\mathrm{f} = 3$ and $p^\mathrm{d} = 1$ such that~\eqref{eq:balancedCondition} reduces to $0 = 2 l^\mathrm{f} +  l^\mathrm{a}$, which has a single solution $(l^\mathrm{f}, l^\mathrm{a}) = (0,0)$, or we have $p^\mathrm{f} = 2$ and $p^\mathrm{d} = 2$ such that~\eqref{eq:balancedCondition} reduces to $1 = 2 l^\mathrm{f} +  l^\mathrm{a}$, which also has a single solution $(l^\mathrm{f}, l^\mathrm{a}) = (0,1)$ (see the last row of Tab.~\ref{tab:balanced})
\end{proof}
\begin{remark}
\label{rem:2views}
For the case of two cameras, we see from~\eqref{eq:balancedCondition} that the number of free and incident lines do not contribute to the parameter count for balanced point-line problems.  In fact, \eqref{eq:balancedCondition} reduces for $m=2$ to $5 = p^\mathrm{f} + p^\mathrm{d}$. Hence, we have the classical minimal problem of recovering five points from two camera images. More precisely, a point-line problem with two cameras is balanced if and only if it has five points. Therefore, it is irrelevant
how many lines are contained in the arrangement or how many points are independent.
There are $5$ combinatorial possibilities to distribute dependent and independent points (see the last row of Tab.~\ref{tab:balanced}).
\end{remark}

\begin{corollary}
There are 39 balanced point-line problems, modulo any number of lines in the case of two views.
They are listed in Tab.~\ref{tab:balanced}.
\end{corollary}
\section{Eliminating world points and lines}\label{sec:equations}
%
\noindent In order to do computations, it is customary to describe problems with implicit equations that do not depend on the world variables. Before we describe such equations, let us phrase the elimination of the world variables geometrically. 

We consider the Zariski closure\footnote{The Zariski closure of a set is the smallest algebraic variety containing the set. 
\ifarxiv
\else
See Sec.~Notation and Concepts in Supplementary Material.
\fi
} of the graph of the joint camera map:
\begin{align*}
\Inc = \overline{\lbrace
(X,C,Y) \in \PplI \times \Cm \times \YplIm \mid}
\hspace{10mm} \\
\overline{Y = \Phi_\PLP(X,C)
\rbrace}.
\end{align*}
The joint camera map $\Phi_\PLP$ is dominant if and only if the projection $\pi_\mathcal{Y}: \Inc \to \YplIm$ onto the last factor is dominant (since this is the projection from the graph of $\Phi_\PLP$ on its codomain). Moreover, the cardinality of the preimage of a generic point $Y \in \YplIm$ under both maps $\Phi_\PLP$ and $\pi_\mathcal{Y}$ is the same.

To make computations simpler, we want to derive the same statement for the following restricted incidence variety, which does not include the 3D structure $\PplI$:
\begin{align*}
\Inc' = \overline{\lbrace
(C,Y) \in  \Cm \times \YplIm \mid}
\hspace{20mm} \\
\overline{ \exists X \in \PplI:  Y = \Phi_\PLP(X,C)
\rbrace}.
\end{align*}
\ifarxiv
We have the following canonical projections:

\begin{tikzcd}
\Inc \arrow[r, "\pi_\mathcal{Y}"]
\arrow[d, "\pi_{\mathcal{C},\mathcal{Y}}" left]
& \YplIm
\\ \Inc' \arrow[ur, "\pi'_\mathcal{Y}" below] &
\end{tikzcd}
where $\pi_{\mathcal{C},\mathcal{Y}}$ omits the first factor and $\pi'_\mathcal{Y}$ projects onto the last factor.
\else
\begin{multicols}{2}
\noindent We have the canonical projections on the right,
where $\pi_{\mathcal{C},\mathcal{Y}}$ omits the first factor and $\pi'_\mathcal{Y}$ projects onto the last factor:
\begin{tikzcd}
\Inc \arrow[r, "\pi_\mathcal{Y}"]
\arrow[d, "\pi_{\mathcal{C},\mathcal{Y}}" left]
& \YplIm
\\ \Inc' \arrow[ur, "\pi'_\mathcal{Y}" below] &
\end{tikzcd}
\end{multicols}
\fi
\begin{lemma} \label{lem:restrictedIncidence}
If $m \geq 2$, a generic point $(C,Y) \in \Inc'$ has a single point in its preimage under $\pi_{\mathcal{C},\mathcal{Y}}$.
\end{lemma}
\begin{proof}
$Y = (x, \ell)$ consists of points $x = (x_{1,1}, \ldots, x_{m,p})$ and lines $\ell = (\ell_{1,1}, \ldots, \ell_{m,l})$ in the $m$ views. Each point $x_{v,i} \in \PP^2$ in a view $v$ is pulled back via the $v$-th camera to a line in $3$-space. As $m \geq 2$, the $m$ pull-back lines for generic\footnote{e.g. no epipoles for two views.} $x_{1,i}, \ldots, x_{m,i}$ intersect in a unique point in $\PP^3$. 
Similarly, each line $\ell_{v,j}$ in a view $v$ is pulled back via the $v$-th camera to a plane in $\PP^3$.
As $m \geq 2$, the $m$ generic\footnote{e.g., no corresponding epipolar lines for two views.} pull-back planes for $\ell_{1,j}, \ldots, \ell_{m,j}$ intersect in a unique line in $\PP^3$.
\end{proof}
\begin{corollary}\label{cor:restrictedIncidence}
A balanced point-line problem $(\PLP)$ is minimal if and only if the projection $\pi'_\mathcal{Y}$ is dominant.
In that case, the degree of the minimal problem is the cardinality of the preimage ${\pi'_\mathcal{Y}}^{-1}(Y)$ of a generic joint image $Y \in \YplIm$ over the complex numbers.
\end{corollary}
\begin{proof}
As we have observed in \emph{Step 2} in Section~\ref{sec:problem-specification}, a balanced point-line problem is minimal if and only if its joint camera map $\Phi_\PLP$ is dominant. This happens if and only if $\pi_\mathcal{Y}$ is dominant. Due to Lemma~\ref{lem:restrictedIncidence}, this is equivalent to that $\pi'_\mathcal{Y}$ is dominant. Similarly, for a generic $Y \in \YplIm$, the cardinalities of the preimages $\Phi_\PLP^{-1}(Y)$, $\pi_\mathcal{Y}^{-1}(Y)$ and ${\pi'_\mathcal{Y}}^{-1}(Y)$ coincide due to Lemma~\ref{lem:restrictedIncidence}.
\end{proof}

We note that it is possible to describe the variety $\Inc'$ as a component of the variety cut out by the equations that we establish in the remainder of this section.

For any instance of a point-line problem, the solutions must satisfy certain equations defined in terms of joint images $(x,\ell) \in \YplIm .$ Our scheme for generating such equations relies on an alternate representation of $(x, \ell)$ defined solely in terms of lines. The equations result from two types of constraints. The first type of constraint is a {\bf line correspondence (LC)}: if $\ell_1, \ldots , \ell_m$ are images of the same world line, with respective homogeneous coordinates $\bl_1, \ldots , \bl_m \in \FF^{3\times 1},$ then 
  \begin{equation}\label{equation:LC}
    \rank \begin{bmatrix} P_1^T\bl_1 & P_2^T\bl_2 & \dots  & P_m^T\bl_m \end{bmatrix} \leq 2.
  \end{equation}
That is, the planes with homogeneous coordinates $P_i^T\bl_i$ share a common line in $\PP^3$. We distinguish two classes of lines in $\PP^2:$

(1) \emph{Visible lines} define valid line correspondences. Besides $m\, l$ observed lines in the joint image, for generic $x$ there is a unique visible line between any two observed points. Taken across all views, any pair of points thus provides a line correspondence which must be satisfied.
  
(2)  Two generic visible lines suffice to define a point. We may use an additional set of (non-corresponding) \emph{ghost lines} to define any points which meet fewer than two visible lines. A generic ghost line contains exactly one observed point --- it is simply a device for generating equations\footnote{``Canonical'' ghost lines, which are rows of 
$\xx{x}$, are often used to eliminate point $x$ from equations by $\xx{c}\,x=0$~\cite{HZ-2003,MaHVKS-IJCV-2004}.}. 

Thus we obtain {\bf common point  (CP)} constraints: given visible and ghost lines $\bl_{v,1}, \ldots \bl_{v, k_i}$ which meet $x_{v,i}$, the projection of the $i$-th point in the view $v\in 1,\dots,m$, we must have
  \begin{equation}\label{equation:CP}
    \rank \begin{bmatrix} P_1^T\bl_{1,1} & \dots & P_m^T\bl_{m,k_i} \end{bmatrix} \leq 3, \quad i=1, \ldots , p.
  \end{equation}

We may encode a point-line problem by specifying some number of visible lines, some number of ghost lines, and which of these lines are incident at each point. We illustrate this encoding with several examples appearing in Figure~\ref{fig:Examples-123}:
\ifarxiv
\begin{figure}[h]
\begin{floatrow}
\ffigbox{%
\centering
\includegraphics[width=0.45\linewidth]{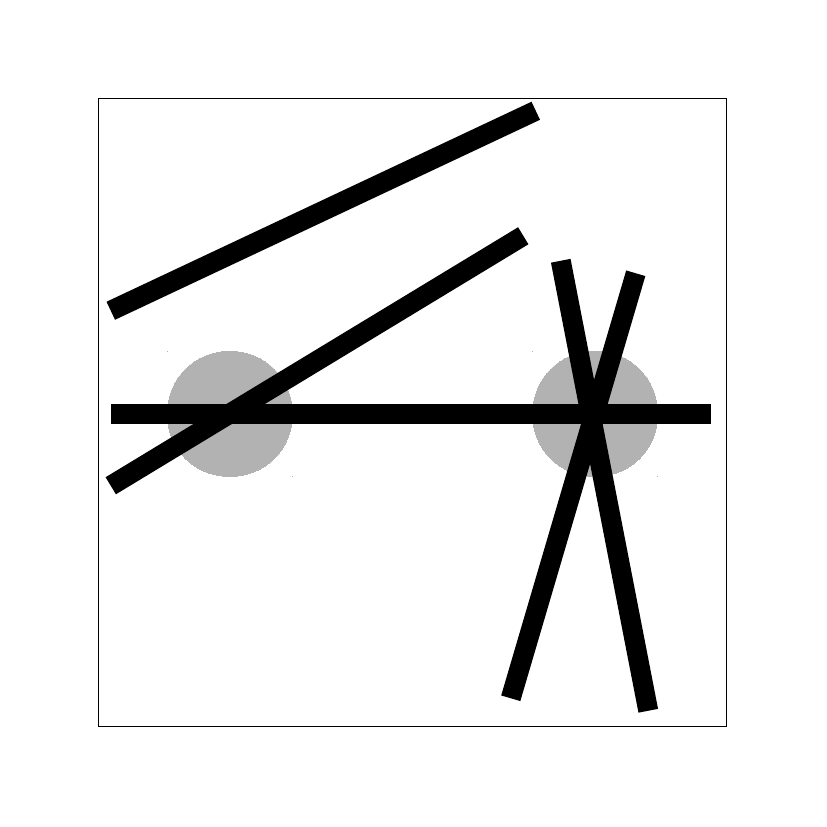}\\[-5pt]Example~(1)\\
\includegraphics[width=0.45\linewidth]{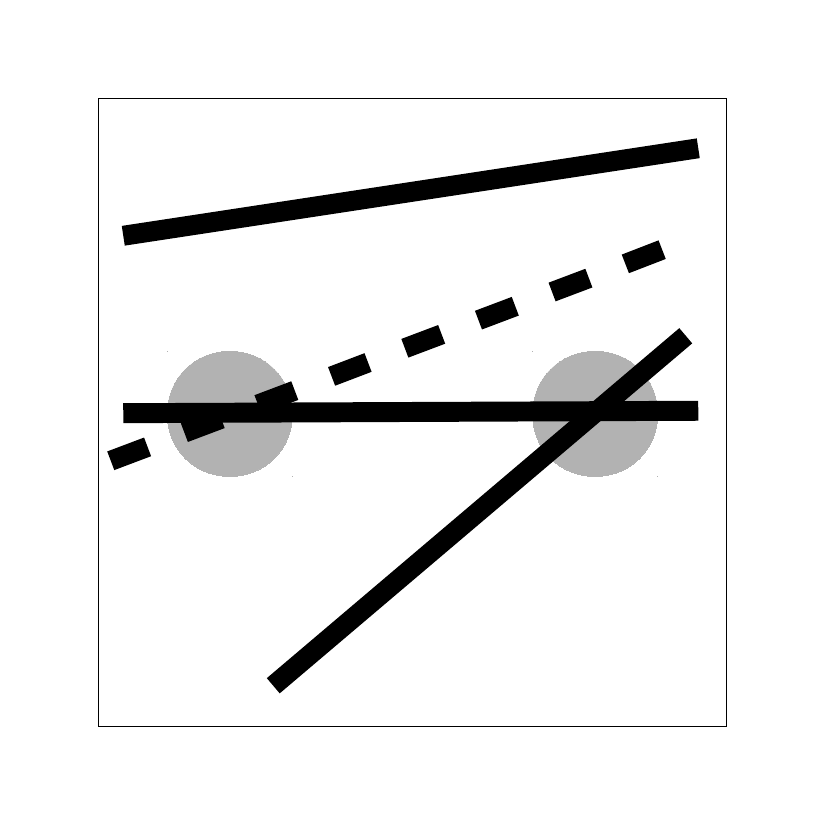}\\[-5pt]Example~(2)\\
\includegraphics[width=0.45\linewidth]{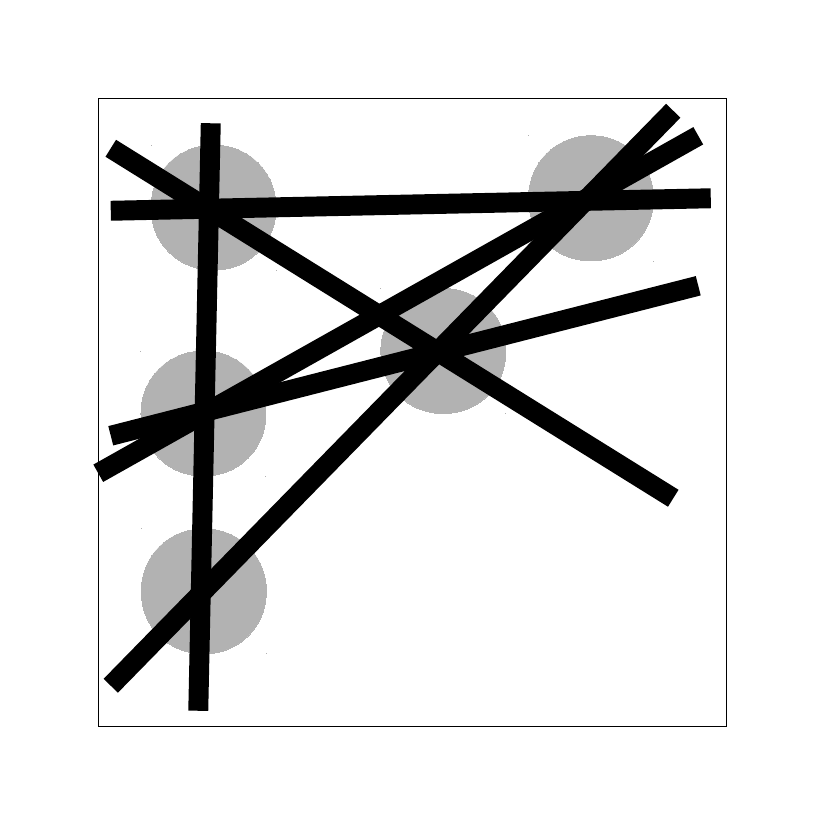}\\[-5pt]Example~(3)}{\caption{Encoding problems with visible and ghost lines.}\label{fig:Examples-123}}
\ffigbox{\includegraphics[width=1\linewidth]{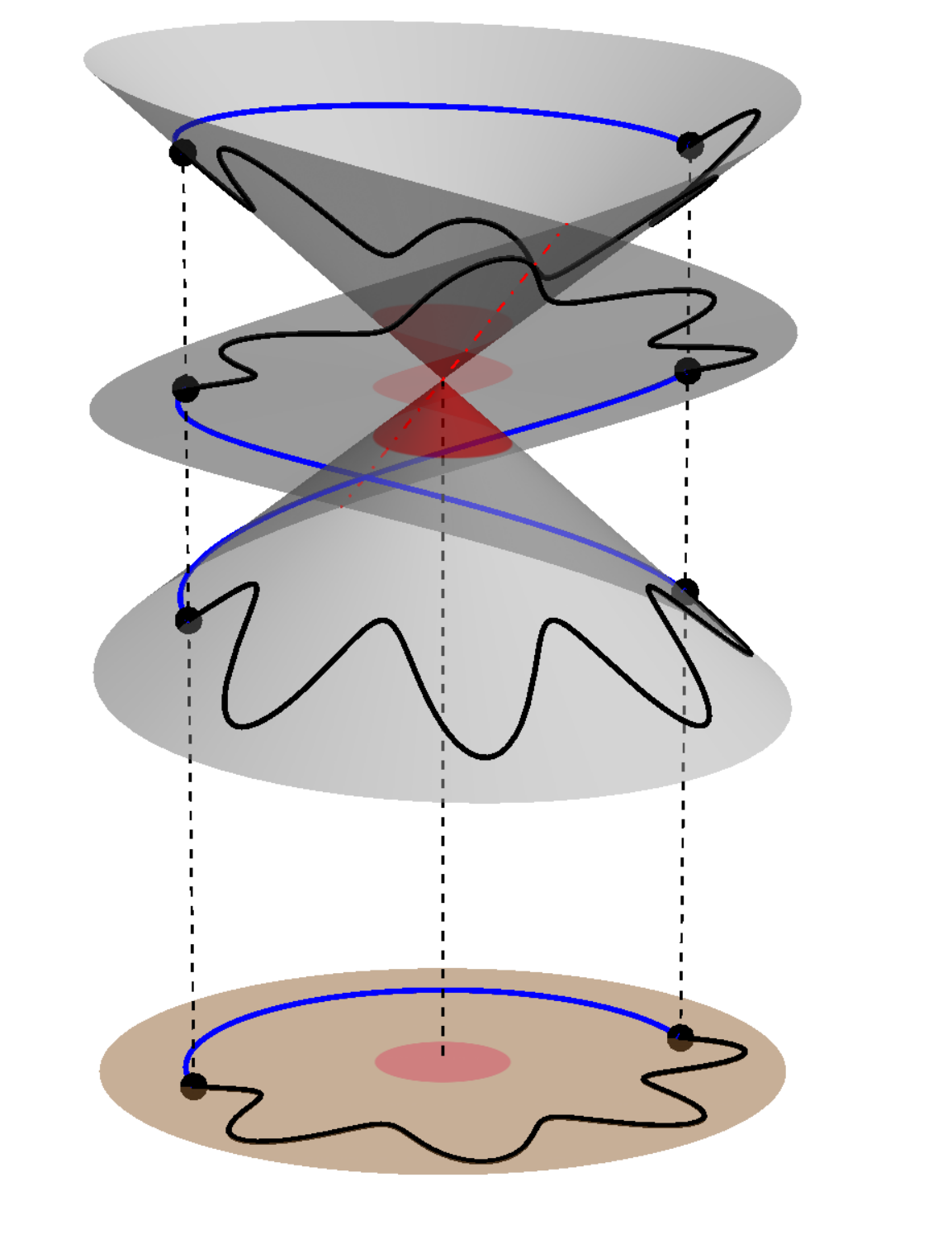}}{\caption{Circulating along the path in the parameter space (downstairs) creates a path meeting all solutions (upstairs).}\label{fig:monodromy}}
\end{floatrow}
\end{figure}
\else
\begin{figure}[tbp]
\begin{floatrow}
\ffigbox{%
\centering
\includegraphics[width=0.45\linewidth]{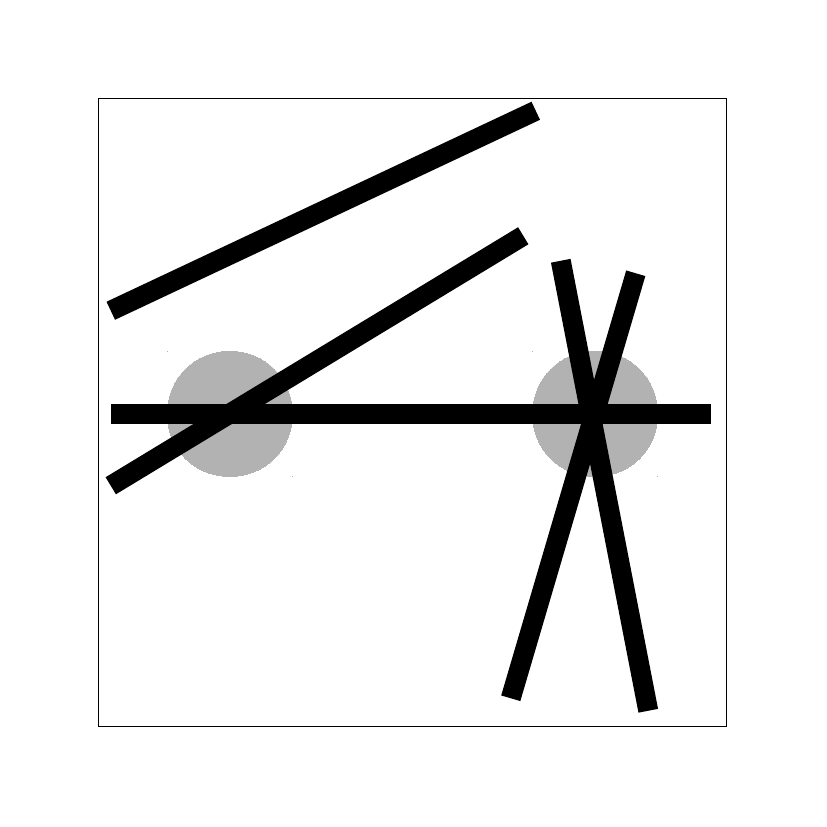}\\[-5pt]Example~(1)\\
\includegraphics[width=0.45\linewidth]{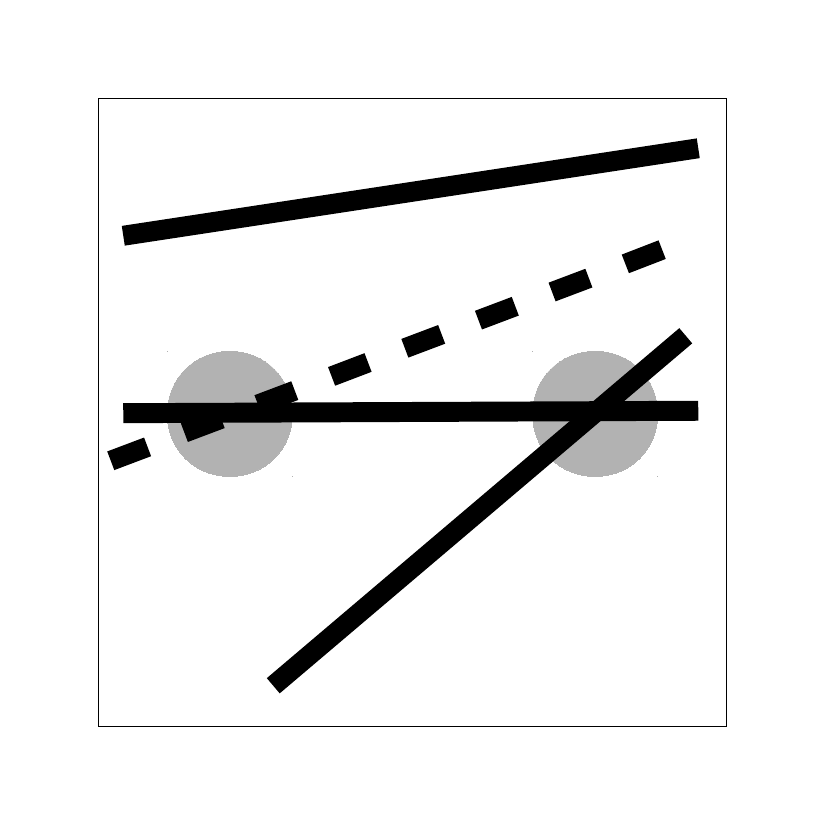}\\[-5pt]Example~(2)\\
\includegraphics[width=0.45\linewidth]{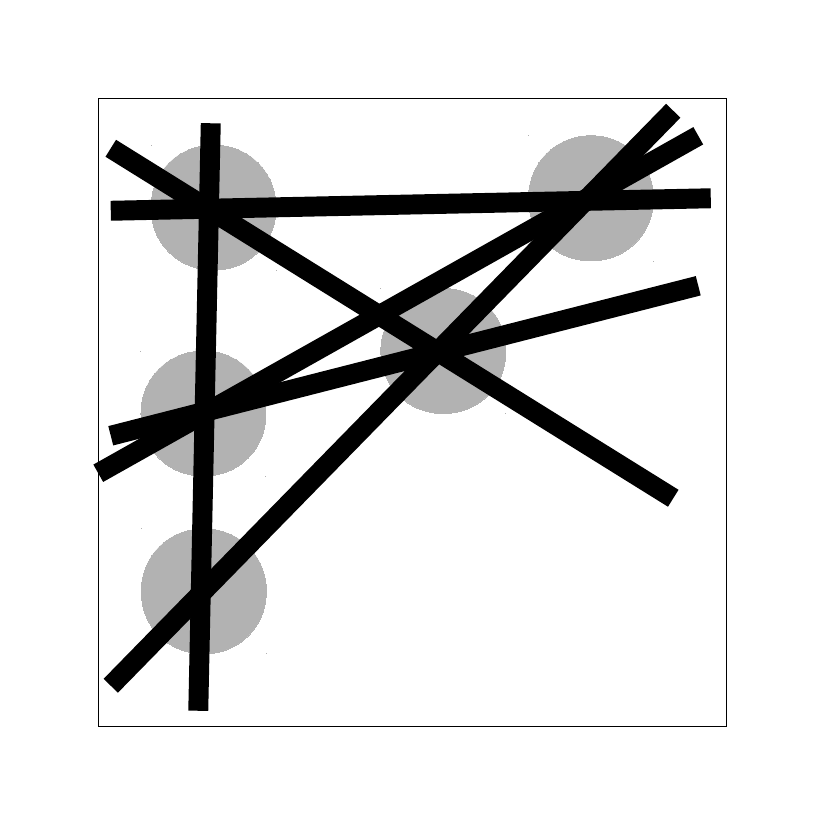}\\[-5pt]Example~(3)}{\caption{Encoding problems with visible and ghost lines.}\label{fig:Examples-123}}
\ffigbox{\includegraphics[width=1.2\linewidth]{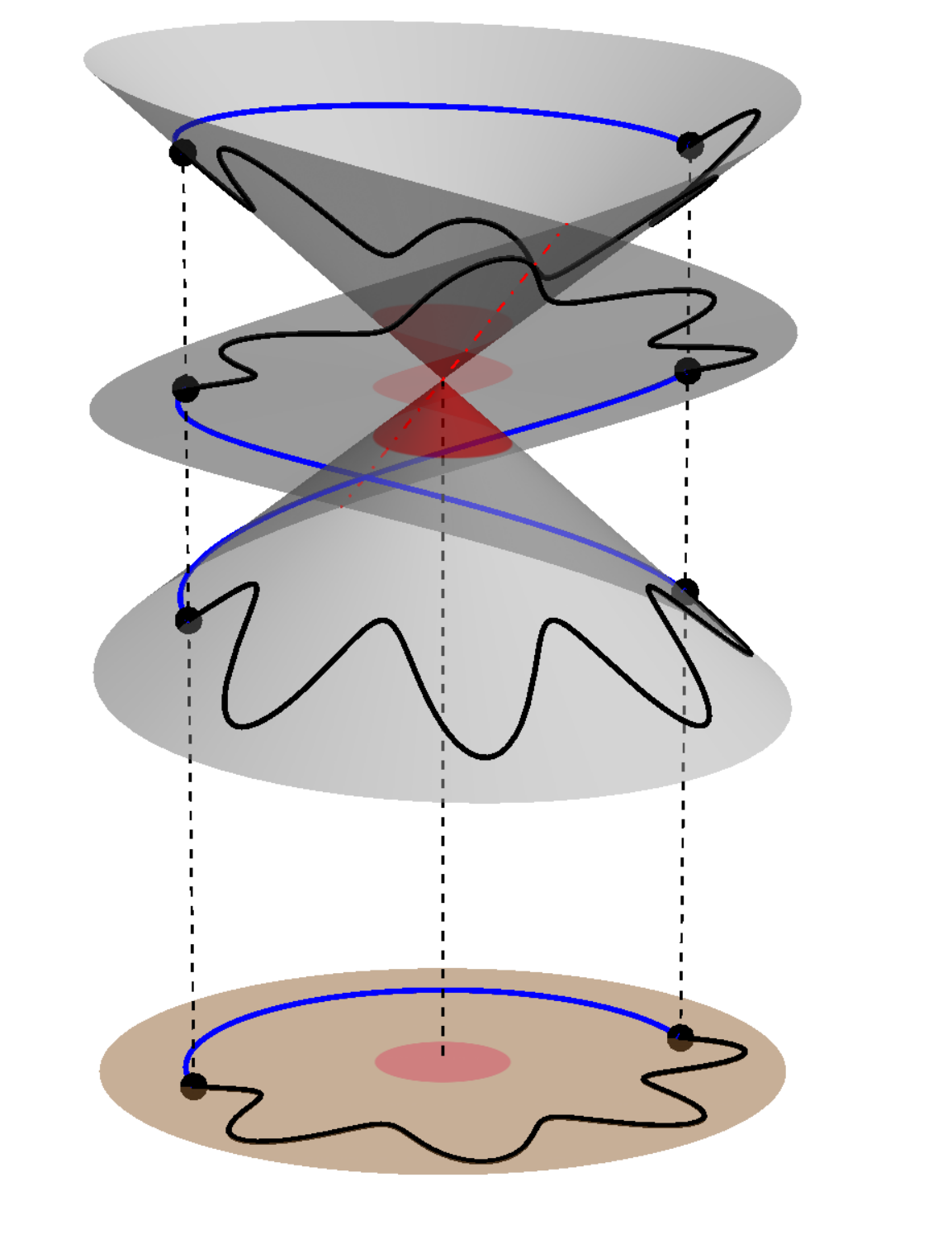}}{\caption{The lifting of the loop in the parameter space (downstairs) meets all solutions (upstairs).}\label{fig:monodromy}}
\end{floatrow}
\end{figure}
\fi
\begin{example}
(1) Consider the point-line problem labeled ``$2013_2$" in Tab.~\ref{tab:balanced}. The lines explicitly drawn in the table together with a visible line between the two free points, as in Fig.~\ref{fig:Examples-123}, suffice to define the scene for generic data.

(2) Consider now the problem labeled ``$2011_1$" in Tab.~\ref{tab:balanced}. The encoding given in Fig.~\ref{fig:Examples-123} includes the given lines, a visible line between the given points, as well as a single ghost line needed to define one of the points. 

(3) Finally, consider the problem labeled ``$3200_3$" in Tab.~\ref{tab:balanced}. The extra visible lines appearing in Fig.~\ref{fig:Examples-123} fix the positions of all points.
\end{example}
LC and CP constraints immediately translate into determinantal conditions:
the $3\times 3$ minors of the matrix in (\ref{equation:LC}) must vanish for each visible line, and the $4\times 4$ minors of the matrix in (\ref{equation:CP}) must vanish for each point. 
Thus, we obtain explicit polynomials for each point-line problem once we fix some encoding and the \emph{cameras parametrization}, \ie a rational map $G : \FF^{6m-7} \ratmap \Cm$. In our computations, we define $G$ via the Cayley parametrization for $\SO(3)$:
\begin{equation}
\label{eq:cayley-parametrization}
    R([a,b,c]) = (I+\xx{[a,b,c]})(I-\xx{[a,b,c]})^{-1}.
\end{equation}
\section{Checking minimality \& computing degrees}\label{sec:degrees-computation}
\noindent Denote by $F = F_{\PLP}$ the system of polynomials resulting from a given point-line problem $(\PLP)$ using our construction of the LC and CP constraints in Sec.~\ref{sec:equations} with the cameras parameterization $G$ plugged in.
\hide{The polynomials define a map
\begin{equation}
  F: \Cm \times \cY \to \FF^N,
\end{equation}
where
\begin{itemize}
\item $\Cm$ is seen as $\FF^{6m-7}$ via the parameterization $G$.
\item $\cY \subset \FF^{3mL}$ is the input space, where $L$ is the total number of visible and ghost lines;
\item $N$ is the number of polynomials---e.g.~the number of minors needed to impose all LC and CP constraints.
\end{itemize}}
The variety of points satisfying $F(C,Y)=0$ contains $\Inc'$ as an irreducible component\footnote{$\Inc'$ cannot be written as a finite union of strictly smaller varieties.}.
\begin{remark}
\label{remark:saturation}  
The variety of points satisfying $F(C,Y)=0$ may also have spurious components corresponding to solutions $(C,Y)$ where the ranks of the matrices \eqref{equation:LC} or~\eqref{equation:CP} are smaller than desired. Such spurious solutions do not correspond to world lines in $\GG_{1,3}$ and must be ruled out. These spurious components are naturally avoided by sampling a point on $\Inc'$. For implicit symbolic calculations, the spurious solutions may be eliminated by including inequations enforcing nonvanishing of the minors of size one smaller. 
\end{remark}
The following algorithm checks minimality of a point-line problem locally; geometrically this amounts to passing to the tangent space of $\Inc'$. 
\begin{algorithm}[Minimal]\label{alg:minimal}\
\begin{algorithmic}[1]
\REQUIRE $(\PLP)$, a balanced point-line problem.
\ENSURE ``Y" if the problem is minimal; ``N" otherwise. 
\STATE $J(C,Y) \gets \frac{\partial F(C,Y))}{\partial C}$
\STATE Take random $C_0\in \Cm$ and random $X_0\in \PplI$.
\STATE $Y_0 \gets \Phi_\PLP(X_0,C_0)$
\RETURN ``Y" if $\, \rank J(C_0,Y_0)=6m-7,$ else ``N." 
\end{algorithmic}
\end{algorithm}
\begin{proof}[Proof of Correctness for Algorithm~\ref{alg:minimal}]
In terminology described in the beginning of Sec.~\ref{sec:equations}, the algorithm checks if the conditions of the Inverse Function Theorem hold at a generic point on $\Inc'$. If they do, the map $\pi'_\mathcal{Y}$ is dominant, since in a neighborhood of $Y_0$ the map has an inverse: i.e if $Y$ is near $Y_0$ then there is $C$ near $C_0$ satisfying $F(C,Y) = 0$. If these conditions do not hold generically, $\pi'_\mathcal{Y}$ is not dominant.  
By Cor.~\ref{cor:restrictedIncidence}, the given point-line problem is minimal if and only if $\pi'_\mathcal{Y}$ is dominant.
\end{proof}
\vspace*{-1ex}
For the minimal problems with two and three views we use the following symbolic algorithm to compute their degrees, i.e. 
the cardinality of the preimage of a generic joint image $Y \in \YplIm$ under the projection $\pi'_\mathcal{Y}$ (by Cor.~\ref{cor:restrictedIncidence}).
\begin{algorithm}[Degree]\label{alg:degree}\ 
  \begin{algorithmic}[1]
    \REQUIRE $(\PLP)$, a point-line minimal problem.
    \ENSURE The degree of this problem. 
    \STATE Take a random $Y_0\in\YplIm$.
    \STATE Compute the \GB\ $B$ of the ideal generated by $F(C,Y_0) \subset \FF[C]$. 
    \RETURN the number of monomials in variables $C$ not divisible by the leading monomials of $B$.
  \end{algorithmic}
\end{algorithm}
\begin{proof}[Proof of Correctness]
  Using \GBs\ to solve a system of polynomial equations is a standard technique in computational nonlinear algebra. Since we are interested only in the solution count, \emph{not} the solutions, we are able to carry out computations relatively quickly; see Remark~\ref{rem:modular}.
\end{proof}
\begin{remark}\label{rem:modular}
  Algorithms~\ref{alg:minimal} and~\ref{alg:degree} are valid over an arbitrary field $\FF$.
  Our main problem is stated over $\QQ$, the rational numbers,  but since the algorithms rely heavily on symbolic techniques such as \GBs\ we use the so-called \emph{modular technique}: we perform computations over a finite field, namely $\FF = \ZZ_p$ for $p<2^{15}$. There is a slight chance that this approach fails for a particular exceptional ``unlucky'' prime $p$, but it is possible to compute the result using several primes and confirm it over $\QQ$ via rational reconstruction.
\end{remark}
\vspace*{-1ex}
Algorithms~\ref{alg:minimal} and~\ref{alg:degree} were implemented and executed in the Macaulay2~\cite{M2} computer algebra system
\footnote{\label{M2footnote}
\ifarxiv
Available at \github.
\else
Available at \github.
\fi
}. 
Due to limitations of \GB \, algorithms we were unable to compute the degrees of any of the problems with $m>3$ with our implementation of Algorithm~\ref{alg:degree}. On the other hand, the degrees of all minimal problems in Tab.~\ref{tab:balanced} are within reach for the \emph{monodromy method}, a technique based on numerical homotopy continuation. Specifically, we follow the \emph{monodromy solver} framework outlined in \cite{Duff-Monodromy} carrying out computation via a Macaulay2 package
\texttt{MonodromySolver}${}^\text{\ref{M2footnote}}$.
Similar techniques have been successfully employed in a number of studies in applied algebraic geometry~\cite{Hauenstein-Identifiability,Kohn-Moment,Breiding-Conics}.  

Imagine the projection  $\pi'_\mathcal{Y} : \Inc' \dashedrightarrow \YplIm$ as the cover map from top to the bottom in~Fig.~\ref{fig:monodromy}. The seed solution $(C_0,Y_0)$ produced as in Algorithm~\ref{alg:minimal} is one of the solutions that project to $Y_0$ at the bottom. Since the Galois group of $\pi'_\mathcal{Y}$ acts transitively on the solutions, one can create enough random paths connecting $Y_0$ and an auxiliary point $Y_1$ so that walking on the liftings of the bottom paths, it is possible to visit all solutions that are above $Y_0$ and, hence, discover the degree. 
\ifarxiv
See \github for code that numerically computes this degree.
\else
See Supplementary Material for how these random paths are created as well as what assurances of completeness this technique provides.
\fi
\section{Conclusion}
We characterized a new class of minimal problems and
discovered problems with small numbers of solutions that call for constructing their efficient solvers~\cite{Fabri-ArXiv-2019}.
\hide{
\section{Conclusion}\label{sec:conclusion}
\noindent We provided a complete classification of generic minimal problems for complete visibility in calibrated cameras. A few of the cases were described before but many new cases were discovered. Several of the new cases possess a small number of solutions and hence call for constructing their efficient solvers~\cite{kukelova2008automatic,Larsson-Saturated-ICCV-2017} and use in practical structure from motion pipelines~\cite{schoenberger2016sfm}.}

\ifarxiv
\paragraph{Acknowledgements} 
We are grateful to ICERM (NSF DMS-1439786 and the Simons Foundation grant 507536) for the hospitality from September 2018 to February 2019, where most ideas for this project were developed. We also thank the many research visitors at ICERM who participated in fruitful discussions on minimal problems.
Research of T.~Duff and A.~Leykin is supported in part by NSF DMS-1719968. T.~Pajdla was supported by the European Regional Development Fund under the project IMPACT (reg. no. CZ.02.1.01/0.0/0.0/15 003/0000468).
\fi
{\small
\bibliographystyle{ieee_fullname}
\bibliography{Pajdla,2019-PAMI-Torii-3D4Loc,Local}}

\begin{thebibliography}{10}\itemsep=-1pt

\bibitem{AgarwalLST17}
Sameer Agarwal, Hon{-}leung Lee, Bernd Sturmfels, and Rekha~R. Thomas.
\newblock On the existence of epipolar matrices.
\newblock {\em International Journal of Computer Vision}, 121(3):403--415,
  2017.

\bibitem{AholtO14}
Chris Aholt and Luke Oeding.
\newblock The ideal of the trifocal variety.
\newblock {\em Math. Comput.}, 83(289):2553--2574, 2014.

\bibitem{Aholt-1107-2875}
Chris Aholt, Bernd Sturmfels, and Rekha~R. Thomas.
\newblock A hilbert scheme in computer vision.
\newblock {\em CoRR}, abs/1107.2875, 2011.

\bibitem{Alismail-odometry}
Hatem~Said Alismail, Brett Browning, and M.~Bernardine Dias.
\newblock Evaluating pose estimation methods for stereo visual odometry on
  robots.
\newblock In {\em the 11th International Conference on Intelligent Autonomous
  Systems (IAS-11)}, January 2011.

\bibitem{Barath-CVPR-2018}
Daniel Barath.
\newblock Five-point fundamental matrix estimation for uncalibrated cameras.
\newblock In {\em 2018 {IEEE} Conference on Computer Vision and Pattern
  Recognition, {CVPR} 2018, Salt Lake City, UT, USA, June 18-22, 2018}, pages
  235--243, 2018.

\bibitem{Barath-TIP-2018}
Daniel Barath and Levente Hajder.
\newblock Efficient recovery of essential matrix from two affine
  correspondences.
\newblock {\em {IEEE} Trans. Image Processing}, 27(11):5328--5337, 2018.

\bibitem{Barath-CVPR-2017}
Daniel Barath, Tekla Toth, and Levente Hajder.
\newblock A minimal solution for two-view focal-length estimation using two
  affine correspondences.
\newblock In {\em 2017 {IEEE} Conference on Computer Vision and Pattern
  Recognition, {CVPR} 2017, Honolulu, HI, USA, July 21-26, 2017}, pages
  2557--2565, 2017.

\bibitem{Breiding-Conics}
Paul Breiding, Bernd Sturmfels, and Sascha Timme.
\newblock 3264 conics in a second.
\newblock {\em arXiv preprint arXiv:1902.05518}, 2019.

\bibitem{Duff-Monodromy}
Timothy Duff, Cvetelina Hill, Anders Jensen, Kisun Lee, Anton Leykin, and Jeff
  Sommars.
\newblock Solving polynomial systems via homotopy continuation and monodromy.
\newblock {\em IMA Journal of Numerical Analysis}, 2018.

\bibitem{Elqursh-CVPR-2011}
Ali Elqursh and Ahmed~M. Elgammal.
\newblock Line-based relative pose estimation.
\newblock In {\em {CVPR} 2011}.

\bibitem{Fabbri-IJCV-2016}
Ricardo Fabbri and Benjamin~B. Kimia.
\newblock Multiview differential geometry of curves.
\newblock {\em International Journal of Computer Vision}, 120(3):324--346,
  2016.

\bibitem{Fabri-ArXiv-2019}
R. Fabri~et al.
\newblock Trifocal relative pose from lines at points and its efficient
  solution.
\newblock {\em Preprint arXiv:1903.09755}, 2019.

\bibitem{Faugeras-IJPRAI-1988}
Olivier~D. Faugeras and Francis Lustman.
\newblock Motion and structure from motion in a piecewise planar environment.
\newblock {\em International Journal of Pattern Recognition and Artificial
  Intelligence}, 2(3):485--508, 1988.

\bibitem{M2}
Daniel~R. Grayson and Michael~E. Stillman.
\newblock Macaulay2, a software system for research in algebraic geometry.
\newblock Available at http://www.math.uiuc.edu/Macaulay2/.

\bibitem{HZ-2003}
Richard Hartley and Andrew Zisserman.
\newblock {\em Multiple View Geometry in Computer Vision}.
\newblock Cambridge, 2nd edition, 2003.

\bibitem{Hartley-IJCV-1997}
Richard~I. Hartley.
\newblock Lines and points in three views and the trifocal tensor.
\newblock {\em International Journal of Computer Vision}, 22(2):125--140, 1997.

\bibitem{Hauenstein-Identifiability}
Jonathan~D. Hauenstein, Luke Oeding, Giorgio Ottaviani, and Andrew~J. Sommese.
\newblock Homotopy techniques for tensor decomposition and perfect
  identifiability.
\newblock {\em Journal f{\"u}r die reine und angewandte Mathematik (Crelles
  Journal)}, 2016.

\bibitem{Holt-PAMI-1995}
Robert~J. Holt and Arun~N. Netravali.
\newblock Uniqueness of solutions to three perspective views of four points.
\newblock {\em {IEEE} Trans. Pattern Anal. Mach. Intell.}, 17(3):303--307,
  1995.

\bibitem{Johansson-ICVGIP-2002}
Bj{\"{o}}rn Johansson, Magnus Oskarsson, and Kalle {\AA}str{\"{o}}m.
\newblock Structure and motion estimation from complex features in three views.
\newblock In {\em {ICVGIP} 2002, Proceedings of the Third Indian Conference on
  Computer Vision, Graphics {\&} Image Processing, Ahmadabad, India, December
  16-18, 2002}, 2002.

\bibitem{JoswigKSW16}
Michael Joswig, Joe Kileel, Bernd Sturmfels, and Andr{\'{e}} Wagner.
\newblock Rigid multiview varieties.
\newblock {\em {IJAC}}, 26(4):775--788, 2016.

\bibitem{Kileel-MPCTV-2016}
Joe Kileel.
\newblock Minimal problems for the calibrated trifocal variety.
\newblock {\em CoRR}, abs/1611.05947, 2016.

\bibitem{Kohn-Moment}
Kathl{\'e}n Kohn, Boris Shapiro, and Bernd Sturmfels.
\newblock Moment varieties of measures on polytopes.
\newblock {\em arXiv preprint arXiv:1807.10258}, 2018.

\bibitem{Kuang-ICCV-2013}
Yubin Kuang and Kalle {\AA}str{\"{o}}m.
\newblock Pose estimation with unknown focal length using points, directions
  and lines.
\newblock In {\em {IEEE} International Conference on Computer Vision, {ICCV}
  2013, Sydney, Australia, December 1-8, 2013}, pages 529--536, 2013.

\bibitem{kukelova2008automatic}
Zuzana Kukelova, Martin Bujnak, and Tomas Pajdla.
\newblock Automatic generator of minimal problem solvers.
\newblock In {\em European Conference on Computer Vision (ECCV)}, 2008.

\bibitem{kukelova2017clever}
Zuzana Kukelova, Joe Kileel, Bernd Sturmfels, and Tomas Pajdla.
\newblock A clever elimination strategy for efficient minimal solvers.
\newblock In {\em Computer Vision and Pattern Recognition (CVPR)}. IEEE, 2017.

\bibitem{larsson2017efficient}
Viktor Larsson, Kalle {\AA}str{\"o}m, and Magnus Oskarsson.
\newblock Efficient solvers for minimal problems by syzygy-based reduction.
\newblock In {\em Computer Vision and Pattern Recognition (CVPR)}, 2017.

\bibitem{Larsson-Syzygy-CVPR-2017}
Viktor Larsson, Kalle {\AA}str{\"{o}}m, and Magnus Oskarsson.
\newblock Efficient solvers for minimal problems by syzygy-based reduction.
\newblock In {\em 2017 {IEEE} Conference on Computer Vision and Pattern
  Recognition, {CVPR} 2017, Honolulu, HI, USA, July 21-26, 2017}, pages
  2383--2392, 2017.

\bibitem{Larsson-Saturated-ICCV-2017}
Viktor Larsson, Kalle {\AA}str{\"{o}}m, and Magnus Oskarsson.
\newblock Polynomial solvers for saturated ideals.
\newblock In {\em {IEEE} International Conference on Computer Vision, {ICCV}
  2017, Venice, Italy, October 22-29, 2017}, pages 2307--2316, 2017.

\bibitem{larsson2017making}
Viktor Larsson, Zuzana Kukelova, and Yinqiang Zheng.
\newblock Making minimal solvers for absolute pose estimation compact and
  robust.
\newblock In {\em International Conference on Computer Vision (ICCV)}, 2017.

\bibitem{Larsson-CVPR-2018}
Viktor Larsson, Magnus Oskarsson, Kalle {\AA}str{\"{o}}m, Alge Wallis, Zuzana
  Kukelova, and Tom{\'{a}}s Pajdla.
\newblock Beyond grobner bases: Basis selection for minimal solvers.
\newblock In {\em 2018 {IEEE} Conference on Computer Vision and Pattern
  Recognition, {CVPR} 2018, Salt Lake City, UT, USA, June 18-22, 2018}, pages
  3945--3954, 2018.

\bibitem{Longuet-Higgins-IVC-1992}
Hugh~Christopher Longuet{-}Higgins.
\newblock A method of obtaining the relative positions of four points from
  three perspective projections.
\newblock {\em Image Vision Comput.}, 10(5):266--270, 1992.

\bibitem{Lowe04IJCV}
David~G. Lowe.
\newblock {Distinctive Image Features from Scale-Invariant Keypoints}.
\newblock {\em International Journal of Computer Vision (IJCV)}, 60(2):91--110,
  2004.

\bibitem{MaHVKS-IJCV-2004}
Yi Ma, Kun Huang, Ren{\'{e}} Vidal, Jana Kosecka, and Shankar Sastry.
\newblock Rank conditions on the multiple-view matrix.
\newblock {\em International Journal of Computer Vision}, 59(2):115--137, 2004.

\bibitem{Malis07-RR6303}
Ezio Malis and Manuel Vargas.
\newblock Deeper understanding of the homography decomposition for vision-based
  control.
\newblock Technical Report 6303, INRIA, 2007.

\bibitem{Matas-ICPR-2002}
Jiri Matas, Step{\'{a}}n Obdrz{\'{a}}lek, and Ondrej Chum.
\newblock Local affine frames for wide-baseline stereo.
\newblock In {\em 16th International Conference on Pattern Recognition, {ICPR}
  2002, Quebec, Canada, August 11-15, 2002.}, pages 363--366, 2002.

\bibitem{Miraldo-ECCV-2018}
Pedro Miraldo, Tiago Dias, and Srikumar Ramalingam.
\newblock A minimal closed-form solution for multi-perspective pose estimation
  using points and lines.
\newblock In {\em Computer Vision - {ECCV} 2018 - 15th European Conference,
  Munich, Germany, September 8-14, 2018, Proceedings, Part {XVI}}, pages
  490--507, 2018.

\bibitem{Nister-5pt-PAMI-2004}
David Nist\'er.
\newblock An efficient solution to the five-point relative pose problem.
\newblock {\em IEEE Transactions on Pattern Analysis and Machine Intelligence},
  26(6):756--770, June 2004.

\bibitem{Nister-CVPR-2007}
David Nist{\'{e}}r, Richard~I. Hartley, and Henrik Stew{\'{e}}nius.
\newblock Using galois theory to prove structure from motion algorithms are
  optimal.
\newblock In {\em 2007 {IEEE} Computer Society Conference on Computer Vision
  and Pattern Recognition {(CVPR} 2007), 18-23 June 2007, Minneapolis,
  Minnesota, {USA}}, 2007.

\bibitem{Nister04visualodometry}
David Nist\'er, Oleg Naroditsky, and James Bergen.
\newblock Visual odometry.
\newblock In {\em Computer Vision and Pattern Recognition (CVPR)}, pages
  652--659, 2004.

\bibitem{Nister-IJCV-2006}
David Nist{\'{e}}r and Frederik Schaffalitzky.
\newblock Four points in two or three calibrated views: Theory and practice.
\newblock {\em International Journal of Computer Vision}, 67(2):211--231, 2006.

\bibitem{Oskarsson-IVC-2004}
Magnus Oskarsson, Andrew Zisserman, and Kalle {\AA}str{\"{o}}m.
\newblock Minimal projective reconstruction for combinations of points and
  lines in three views.
\newblock {\em Image Vision Comput.}, 22(10):777--785, 2004.

\bibitem{QuanTM2006}
Long Quan, Bill Triggs, and Bernard Mourrain.
\newblock Some results on minimal euclidean reconstruction from four points.
\newblock {\em Journal of Mathematical Imaging and Vision}, 24(3):341--348,
  2006.

\bibitem{Raguram11IJCV}
Rahul Raguram, Changchang Wu, Jan-Michael Frahm, and Svetlana Lazebnik.
\newblock {Modeling and Recognition of Landmark Image Collections Using Iconic
  Scene Graphs}.
\newblock {\em International Journal of Computer Vision (IJCV)},
  95(3):213--239, 2011.

\bibitem{rocco2018neighbourhood}
Ignacio Rocco, Mircea Cimpoi, Relja Arandjelović, Akihiko Torii, Tomas Pajdla,
  and Josef Sivic.
\newblock Neighbourhood consensus networks, 2018.

\bibitem{SalaunMM-ECCV-2016}
Yohann Sala{\"{u}}n, Renaud Marlet, and Pascal Monasse.
\newblock Robust and accurate line- and/or point-based pose estimation without
  manhattan assumptions.
\newblock In {\em {ECCV} 2016}.

\bibitem{Sattler-PAMI-2017}
Torsten Sattler, Bastian Leibe, and Leif Kobbelt.
\newblock Efficient {\&} effective prioritized matching for large-scale
  image-based localization.
\newblock {\em {IEEE} Trans. Pattern Anal. Mach. Intell.}, 39(9):1744--1756,
  2017.

\bibitem{schoenberger2016sfm}
Johannes~Lutz Sch\"{o}nberger and Jan-Michael Frahm.
\newblock Structure-from-motion revisited.
\newblock In {\em Conference on Computer Vision and Pattern Recognition
  (CVPR)}, 2016.

\bibitem{Snavely-SIGGRAPH-2006}
N. Snavely, S.~M. Seitz, and R. Szeliski.
\newblock {Photo tourism: exploring photo collections in 3D}.
\newblock In {\em ACM SIGGRAPH'06}, 2006.

\bibitem{snavely2008modeling}
Noah Snavely, Steven~M. Seitz, and Richard Szeliski.
\newblock Modeling the world from internet photo collections.
\newblock {\em International Journal of Computer Vision (IJCV)},
  80(2):189--210, 2008.

\bibitem{Stewenius-ISPRS-2006}
Henrik Stew\'enius, Christopher Engels, and David Nist\'er.
\newblock Recent developments on direct relative orientation.
\newblock {\em ISPRS J. of Photogrammetry and Remote Sensing}, 60:284--294,
  2006.

\bibitem{svarm2017city}
Linus Sv{\"a}rm, Olof Enqvist, Fredrik Kahl, and Magnus Oskarsson.
\newblock City-scale localization for cameras with known vertical direction.
\newblock {\em IEEE transactions on pattern analysis and machine intelligence},
  39(7):1455--1461, 2017.

\bibitem{taira2018inloc}
Hajime Taira, Masatoshi Okutomi, Torsten Sattler, Mircea Cimpoi, Marc
  Pollefeys, Josef Sivic, Tomas Pajdla, and Akihiko Torii.
\newblock {InLoc}: Indoor visual localization with dense matching and view
  synthesis.
\newblock In {\em Computer Vision and Pattern Recognition (CVPR)}, 2018.

\bibitem{Trager-PhD-2018}
Matthew Trager.
\newblock {\em Cameras, Shapes, and Contours: Geometric Models in Computer
  Vision. (Cam{\'{e}}ras, formes et contours: mod{\`{e}}les
  g{\'{e}}om{\'{e}}triques en vision par ordinateur)}.
\newblock PhD thesis, {\'{E}}cole Normale Sup{\'{e}}rieure, Paris, France,
  2018.

\bibitem{Ponce-IJCV-2016}
Matthew Trager, Jean Ponce, and Martial Hebert.
\newblock Trinocular geometry revisited.
\newblock {\em International Journal Computer Vision}, pages 1--19, March 2016.

\bibitem{Zhao-PAMI-2019}
Ji {Zhao}, Laurent {Kneip}, Yijia {He}, and Jiayi {Ma}.
\newblock Minimal case relative pose computation using ray-point-ray features.
\newblock {\em IEEE Transactions on Pattern Analysis and Machine Intelligence},
  pages 1--1, 2019.

\end{thebibliography}
\end{document}